\documentclass{article}

\usepackage{hyperref}
\hypersetup{
colorlinks=true,citecolor=green!60!black,linkcolor=red!60!black,
pagebackref=true
}
\usepackage[table]{xcolor}
\usepackage{graphicx}
\usepackage{times}
\usepackage{graphicx}
\usepackage{array}
\usepackage{comment}
\usepackage{latexsym}
\usepackage{amssymb}
\usepackage{float}
\usepackage{amsmath}
\usepackage{tikz}
\usetikzlibrary{decorations.pathreplacing,calc}
\usepackage{amsthm}
\usepackage{graphicx}
\usepackage{pst-node}
\usepackage{color}
\usepackage[numbers]{natbib}
\usepackage[noend, boxed, noline]{algorithm2e}
\usepackage{paralist}
\usepackage[]{units}
\usepackage{chngcntr}
\usepackage{apptools}
\usepackage{mathrsfs}
\AtAppendix{\counterwithin{proposition}{section}}

\allowdisplaybreaks

\newtheorem{definition}{Definition}
\newtheorem{proposition}{Proposition}
\newtheorem{observation}{Observation}

\newtheorem{theorem}{Theorem}
\newtheorem{lemma}{Lemma}
\newtheorem{problem}{Problem}
\theoremstyle{definition}
\newtheorem{example}{Example}

\newcommand{\Omit}[1]{}   
\newcommand{\fj}[1]{}   
\newcolumntype{L}[1]{>{\raggedright\let\newline\\\arraybackslash\hspace{0pt}}m{#1}}
\newcolumntype{C}[1]{>{\centering\let\newline\\\arraybackslash\hspace{0pt}}m{#1}}
\newcolumntype{R}[1]{>{\raggedleft\let\newline\\\arraybackslash\hspace{0pt}}m{#1}}
\newcommand{\fjl}[1]{}   

\newcommand{\wone}{{\mathrm{W[1]}}}
\newcommand{\fpt}{{\mathrm{FPT}}}
\newcommand{\harmonic}{{\mathrm{H}}}
\newcommand{\calA}{{{\mathcal{A}}}}
\newcommand{\argmax}{{{\mathrm{argmax}}}}
\newcommand{\calR}{{{\mathcal{R}}}}
\newcommand{\integers}{{{\mathbb{Z}}}}
\newcommand{\naturals}{{{\mathbb{N}}}}
\newcommand{\reals}{{{\mathbb{N}}}}
\newcommand{\rationals}{{{\mathbb{Q}}}}

\makeatletter
\newtheorem*{rep@theorem}{\rep@title}
\newcommand{\newreptheorem}[2]{%
\newenvironment{rep#1}[1]{%
 \def\rep@title{\autoref{##1}}%
 \begin{rep@theorem}}%
 {\end{rep@theorem}}}
\makeatother
\newreptheorem{theorem}{Theorem}
\newreptheorem{proposition}{Proposition}
\newreptheorem{lemma}{Lemma}
\newreptheorem{corollary}{Corollary}

\title{Multi-Attribute Proportional Representation\footnote{The preliminary version of this paper was presented at the \emph{30th Conference on Artificial Intelligence (AAAI-2016)}.}}

\date{}

\author{
J\'{e}r\^{o}me Lang\\
       {Universit\'e Paris-Dauphine}\\
       {Paris, France}\\
\and
Piotr Skowron\\
	University of Warsaw\\
	Warsaw, Poland\\
}

\begin{document}

\maketitle

\begin{abstract}
We consider the following problem in which a given number of items has to be chosen from a predefined set. Each item is described by a vector of attributes and for each attribute there is a desired distribution that the selected set should have.
We look for a set that fits as much as possible the desired distributions on all attributes.
An example of application is the choice of members for a representative committee, where candidates are described by attributes such as gender, age and profession, and where we look for a committee that for each attribute offers a certain representation, i.e., a single committee that contains a certain number of young and old people, certain number of men and women, certain number of people with different professions, etc. Another example of application is the selection of a common set of items to be used by a group of users, where items are labelled by attribute values. 
With a single attribute the problem collapses to the apportionment problem for party-list proportional representation systems (in such a case the value of the single attribute would be a political affiliation of a candidate).
We study the properties of the associated subset selection rules, as well as their computational complexity.
\end{abstract}

\section{Introduction}\label{intro}

Consider the following example.
A research department has to choose $k$ members for a recruiting committee. A selected committee should be gender-balanced, ideally containing 50\% of male and 50\% of female. Additionally, a committee should represent different research areas in certain proportions: ideally it should contain 55\% of researchers specialising in area $1$, 25\% of experts in area $2$, and 20\% in area $3$. Another requirement is that the committee should contain 30\% of junior and 70\% of senior researchers, and finally, the repartition between local and external members should be kept in proportions 30\% to 70 \%. The pool of candidates from which the department can select members of such a committee is the following:

\begin{center}
\begin{tabular}{c|cccc}
Name & Gender & Group & Age & Affiliation\\ \hline
Ann & $F$ & $1$ & $J$ & $L$\\
Bob & $M$ & $1$ & $J$ & $E$\\
Charlie & $M$ & $1$ & $S$ & $L$ \\
Donna & $F$ & $2$ & $S$ & $E$\\
Ernest & $M$ & $1$ & $S$ & $L$\\
George & $M$ & $1$ & $S$ & $E$\\
Helena & $F$ & $2$ & $S$ & $E$\\
John & $M$ & $2$ & $J$ & $E$\\
Kevin & $M$ & $3$ & $J$ & $E$\\
Laura & $F$ & $3$ & $J$ & $L$\\
\end{tabular}
\end{center}

In the given example, if the department wants to select $k = 3$ members, then it is easy to see that there exists no committee that would satisfy all the criteria perfectly. Nevertheless, some committees are better than others: intuitively we feel that in the selected committee the ratio of the numbers of members representing different genders should be either equal to 2:1 or to 1:2, the ratio of the numbers of members representing areas $1$, $2$ and $3$, should be equal to 2:1:0. Further, the selected committee should contain one junior and two senior members, and exactly one member of the selected committee should have local affiliation. Such relaxed criteria can be achieved by selecting Ann, Donna, and George. Now, let us consider the above example for the case when $k = 4$. In such a case, the ideal ratios between the numbers of members for each of the four attributes should be equal to 1:1, 2:1:1, 1:3, and 1:3, respectively. Observe, however, that there exists no committee satisfying such relaxed criteria. According to different criteria, in this case the best committee can be for instance \{Ann, Charlie, Donna, George\}, with two externals instead of three, or \{Charles, Donna,  George, Kevin\}, with males being over-represented. 

In this paper we formalise the intuition given in the above example and we define what it means for a committee to be optimal, with respect to multi-attribute proportional representation. In our approach we leverage classical tools from political and social sciences, in particular we adapt the concept of \emph{proportional apportionment} from the political science literature~\cite{Balinski01} to the case of multiple attributes. The central question of the apportionment problem is how to distribute parliament seats between political parties, given the numbers of votes cast for each party. Indeed, we can consider our multi-attribute problem, with the single attribute being a political affiliation of a candidate, and the desired distributions being the proportions of votes cast for different parties. In such a case we can see that selecting a committee in our multi-attribute proportional representation system boils down to selecting a parliament according to some apportionment criterion.

To emphasise the analogy between our model and the apportionment methods, we should provide some discussion on where the desired proportions for attributes come from.
Typically, but not always, they come from {\em votes}. For instance, each voter might give her preferred value for each attribute, and the ideal proportions coincide with the observed frequencies. For instance, out of 20 voters, 10 would have voted for a male and 10 for a female, 13 for a young person and 7 for a senior one, etc.\footnote{How to aggregate in a consistent way ideal proportions specified by different voters is a nontrivial problem addressed in \cite{ConitzerFBL16}.}  
 It is worth mentioning that the voters might cast approval ballots, that is for each attribute they might define a set of approved values rather than pointing out the single most preferred one. On the other hand, sometimes, instead of votes, there are ``global'' preferences on the composition of the committee, expressed directly by the group, imposed by law, or by other constraints that should be respected as much as possible independently of voters' preferences.

There is a variety of apportionment methods considered in the literature
(we refer the reader to the survey of Balinski and Young~\cite{Balinski01}). They are evaluated by means of properties; among those that are deemed important and have been extensively studied in the literature, we find 
 \emph{non-reversal}, \emph{respect of quota}, \emph{population monotonicity}, and \emph{house monotonicity} (see~\cite{balinski1979criteria}).
We define the analogs of these properties for the multi-attribute domain. These properties give us some insights into the nature of multi-attribute committee selection mechanisms; in particular, their analysis allows us to view certain selection methods as generalisations of the appropriate apportionment rules. Specifically, following this approach, in this paper we define multi-attribute variants of the Hamilton rule and of the d'Hondt rule of apportionment, hereinafter referred to as the \emph{multi-attribute Hamilton rule} and the \emph{multi-attribute d'Hondt rule}.

The multi-attribute case, however, is also substantially different from the single-attribute one. In particular, multi-attribute proportional representation systems exhibit computational problems that do not appear in the single-attribute setting. Indeed, in the second part of our paper we show that finding an optimal committee is often {\sf NP}-hard. However, we show that this challenge can be addressed by designing efficient approximation and fixed-parameter tractable algorithms. In particular, the core technical contribution of this paper lies in the analysis of approximation guarantees provided by the local-search algorithm for the problem of finding an optimal committee, with respect to a certain measure of multi-attribute proportional representation.

We believe that the model formalised in this paper has broad applications. As an example, consider a political system where the voters do not vote for the candidates directly, but rather for their opinions on various issues. For instance, quoting Lang and Xia~\cite{LangXia15}, in 2012, voters in California had to decide in simultaneous multiple referenda whether to adopt each of the given eleven propositions~\footnote{http://en.wikipedia.org/wiki/California\_elections,\_November\_ 2012}; a similar vote also took place in Florida. Given that the voters vote on propositions, our algorithms can be used to find a set of candidates that, in some sense, best represents opinions of voters about propositions. The number of propositions can be even larger: for instance, political parties have usually quite elaborate programs in which they refer to tens or hundreds of issues.

Further, our algorithms can be useful for selecting diversified groups of people. For instance, assume that our goal is to prepare an advertisement campaign. In such a case it is often desirable to depute this task to a team where men, women, people with different age and different education level are well represented. Similarly, when we select a jury we would like it to be representative according to different criteria, such as ethnicity, gender, age, religious beliefs, education level and wealth. Admitting PhD students is another example, where we would like to have a diversity with respect to ethnicity, gender, nationality, but also with respect to skills, education background, or disciplines of interest.

As another example, consider a library offering a set of movies to buy. In ImDB\footnote{http://www.imdb.com/} movies can be described by many attributes, such as genre, country, language, year, actors, directors, awards, etc. Users often look for movies by their attributes. Our algorithms can help such library to find a representative collection of movies that fits the collective will as much as possible. Finding a representative collective set of attribute-value items can be also used as a tool for implementing {\em group recommendations}~\cite{ARCDY09}, where the goal is to recommend a set of items for a group of agents, based on their (possibly conflicting) preferences: in some recent approaches to  group recommendation (see \cite{GarciaPSO12} for a survey and comparison of four approaches), each item is seen as a set of features, users' preferences over features are elicited, and the aim of the system is to suggest a few representative items, such as set of movies or a set of tourist activities that comply as much as possible with the users' preferences over features. 

This paper is organised as follows. In \autoref{sec:apportionment} we recall some useful concepts and definitions relating to methods of apportionment.
We present our model in \autoref{model} and in \autoref{sec:multiAttributeRules} we introduce two different optimisation criteria and define
multi-attribute committee selection rules optimising these criteria. 
In \autoref{computation} we show that, although computation of optimal committees is generally {\sf NP}-hard, 
there exist good approximation and fixed-parameter tractable algorithms for finding them. 
We position our work with respect to related areas in \autoref{related}.
In \autoref{discussion} we give a detailed discussion on the model and some of its possible extensions.
Finally, in \autoref{conclu} we conclude and point to further research issues.

\section{Preliminaries: Methods of Apportionment}\label{sec:apportionment}
For each integer $i \in \naturals$, by $[i]$ we denote the set of the first $i$ natural numbers, $[i] = \{1, \ldots, i\}$.

Consider a sequence of $t$ political parties, denoted as $P_1, \ldots, P_t$. For each $i \in [t]$, let $v_i$ denote the number of votes given to party $P_i$.
An apportionment rule is a method that given a distribution of votes among parties, $v = (v_1, \ldots, v_t)$ where $v_i$ denotes the number of votes cast for party $P_i$, and the number of seats $h$ (the size of the house), returns a distribution of the $h$ seats among the $t$ parties.
We denote the number of seats allocated to party $P_i$ by $r_i$.

As is often the case in social choice, ties may occur and we have to choose between resoluteness and neutrality between parties: a resolute apportionment rule returns a single solution by sacrificing neutrality in case a tie occurs, and an irresolute apportionment rule returns all tied apportionments. In the rest of the paper we focus on resolute rules, and assume that ties are broken by an exogenous priority relation between parties. All our results are easily adaptable to irresolute rules.  

Formally, an apportionment rule is a function $\calA\colon \naturals^t \times \naturals \to \naturals^t$ that for each $v \in \naturals^t$ and each $h \in \naturals$ 
returns a vector $\calA(v, h) = (r_1, \ldots, r_t)$ satisfying the following two conditions:
\begin{inparaenum}[(i)]
\item $\sum_{i \in [t]} r_i = h$,
\item $r_i \in \naturals \cup \{0\}$ for each $i \in [t]$.
\end{inparaenum}
We will use the symbol $v_{+}$ to denote the sum of all votes, $v_{+} = \sum_{i=1}^t v_i$. 

There are numerous apportionment rules considered in the literature. The two most commonly-used classes of apportionment rules are the {\em largest remainder} and the {\em divisor} methods \cite{Balinski01}, which we briefly describe below.  

\subsection{Largest Remainder Methods} 

The following definition describes one of the most prominent classes of apportionment methods.

\begin{definition}[The largest remainder methods.]
Let $q \in \rationals$ be a rational number. The largest remainder method with quota $q$ works in two steps. In the first step, each party $P_i$ is allocated $\lfloor \nicefrac{v_i}{q} \rfloor$ seats (the quota value must be chosen in such a way that the number of seats allocated in the first step is guaranteed to be between $h-t$ and 
$h$). In the second step, the remaining seats are allocated to the parties so that each party is allocated either one or zero additional seats. The parties which are allocated an additional seat are the ones with the largest values of the remainders $\nicefrac{v_i}{q} - \lfloor \nicefrac{v_i}{q} \rfloor$ (using the tie-breaking priority relation if necessary).
\end{definition}

The most common choice of a quota is the {\em Hare quota}, defined as $q_{\mathrm{Hare}} = \nicefrac{v_{+}}{h}$; the method based on the Hare quota is called the {\em Hamilton method} (also known as the \emph{largest remainder method} or \emph{Hare-Niemeyer method}).\footnote{Other common choices are the {\em Droop quota} $1+\frac{v_{+}}{1+h}$, the {\em  Hagenbach-Bischoff quota} $\frac{v_{+}}{1+h}$ and the {\em Imperiali quota} $\frac{v_{+}}{2+h}$.} The Hamilton method was one of the first methods used in the contemporary democracies. Its definition dates back to the 18th century and it was first used to select the members of the U.S. House of Representatives between 1852 and 1900. Currently, with slight modifications, it is used in parliamentary elections in Russia, Ukraine, Tunisia, Namibia, and Hong Kong. Below we provide an example illustrating the Hamilton method.

\begin{example} \label{ex1}
Consider the instance with four parties and $100$ voters. Assume that 4, 12, 33, and 51 votes were cast for parties $P_1$, $P_2$, $P_3$, and $P_4$, respectively. Let us set $h = 10$, thus $q_{\mathrm{Hare}} = \nicefrac{v_{+}}{h} = 10$. In the first step the parties $P_1$, $P_2$, $P_3$, and $P_4$ are allocated 0, 1, 3, and 5 seats, respectively. The remainders for the four parties equal to $\nicefrac{4}{10}$, $\nicefrac{2}{10}$, $\nicefrac{3}{10}$, and $\nicefrac{1}{10}$, respectively. In the second step, the single remaining seat goes to the party with the highest reminder, i.e., to $P_1$. Consequently, the allocation of the seats returned by the Hamilton method is given by the vector $(1, 1, 3, 5)$. \qed
\end{example}

\subsection{Divisor Methods} 

Divisor methods (also known as \emph{highest average methods}) constitute another class of common and important apportionment methods.

\begin{definition}[Divisor methods.]
Let $d = (d_1, d_2, \ldots)$ be a nondecreasing sequence of positive values. The divisor method defined by sequence $d$ starts with an empty allocation $(0, \ldots, 0)$, and in each of the $h$ consecutive steps assigns one additional seat to some party. Let $s_i(j)$ denote the number of seats allocated to party $P_j$ just before step $i$. In the $i$-th step the party $P_j$ with the highest ratio $\nicefrac{v_j}{d_{s_i(j)+1}}$ is allocated an additional seat (using the tie-breaking priority relation if necessary). We denote this party as 
$A(v,h,i)$.
\end{definition}

The most commonly used sequences of divisors are $d_{\mathrm{DHondt}} = (1, 2, 3, \ldots)$ and $d_{\mathrm{SL}} = (1, 3, 5, \ldots)$. The divisor method based on the sequence $d_{\mathrm{DHondt}}$ is called the \emph{d'Hondt method} (it is also known as the \emph{Jefferson method} or the \emph{Hagenbach--Bischoff method}). The definition of the d'Hondt method dates back to the 18th century as well, and it is currently used for apportionment in more than 40 countries. The divisor method based on the sequence $d_{\mathrm{SL}}$ is known as the Sainte-Lagu\"e method (sometimes referred to as the \emph{Webster method}, \emph{Schepers method}, or the \emph{method of major fractions}) and is currently used in several countries.

\begin{example} \label{ex2}
Consider the instance from \autoref{ex1}. The below table shows the computation of the d'Hondt method. In the $i$-th iteration the $i$-th highest value from the table is selected and a seat is allocated to the party that corresponds to this value. For instance, the first seat will be allocated to party $P_4$, which corresponds to the highest value of $51$. The highest $10$ values are shown in bold font: these are the values that correspond to the $10$ seats allocated to parties. 
\begin{align*}
\begin{array}{c|cccc}
  & v_1  & v_2  & v_3  & v_4  \\
  \hline 
\nicefrac{v_i}{1} & 4 & \mathbf{12} & \mathbf{33} & \mathbf{51} \\
\nicefrac{v_i}{2} & 2 & 6 & \mathbf{16.5} & \mathbf{25.5} \\
\nicefrac{v_i}{3} & 1.33 & 4 & \mathbf{11} & \mathbf{17} \\
\nicefrac{v_i}{4} & 1 & 3 & 8.25 & \mathbf{12.75} \\
\nicefrac{v_i}{5} & 0.8 & 2.4 & 6.6 & \mathbf{10.2} \\
\nicefrac{v_i}{6} & 0.66 & 2 & 5.5 & \mathbf{8.5} \\
\nicefrac{v_i}{7} & 0.57 & 1.71 & 4.71 & 7.28
\end{array}
\end{align*}
According to the d'Hondt method the following parties will be allocated consecutive seats: we start by giving a seat to $P_4$ (that is, $A(v,10,1) = P_4)$, because the largest value in the table is $\nicefrac{v_4}{1} = 51$; then a seat to $P_3$, because the second largest value is $\nicefrac{v_3}{1} = 33$; then a second seat to $P_4$, because  the third largest value is $\nicefrac{v_4}{2} = 25.5$; then a third seat to $P_4$, and then $P_3$, $P_4$, $P_2$, $P_3$, $P_4$, $P_4$. In the end, parties $P_1$, $P_2$, $P_3$, and $P_4$ will get 0, 1, 3, and 6 seats, respectively. \qed
\end{example}

\subsection{Properties of Methods of Apportionment}\label{sec:apportionmentProperties}

Several properties of apportionment methods have been studied, starting with Balinski and Young~\cite{balinski1979criteria}. Below, we recall the definitions of the several of them, which will be useful in our further discussion. Recall that $v$ denotes the vector of votes, $v_{+}$ denotes the total number of all votes, $h$ is the number of available seats and that $(r_1, \ldots, r_t) = \calA(v, h)$.
\begin{description}
\item[Non-reversal.] The rule $\calA$ is said to satisfy \emph{non-reversal} if for each parties $P_i$, $P_j$, $r_i \geq r_j$ holds whenever $v_i > v_j$.

\item[Respect of quota.] The rule $\calA$ is said to \emph{respect quota} if for each party $P_i$ it holds that $\lfloor \nicefrac{v_ih}{v_+} \rfloor \leq r_i \leq \lceil \nicefrac{v_ih}{v_+} \rceil$.

\item[Party population monotonicity.] Consider two vectors of votes $v = (v_1, \ldots, v_t)$ and $v' = (v_1', \ldots, v_t')$ and a party $P_i$ such that:
\begin{inparaenum}[(i)]
\item $\nicefrac{v_i}{v_+} > \nicefrac{v_i'}{v_+'}$, and
\item $\nicefrac{v_j}{v_{\ell}} = \nicefrac{v_j'}{v_{\ell}'}$ for each $j, \ell \neq i$.
\end{inparaenum}
The rule $\calA$ satisfies \emph{party population monotonicity} if for each such vectors of votes $v$ and $v'$, it holds that $r_i \geq r_i'$, where $(r_1, \ldots, r_t) = \calA(v, h)$ and $(r_1', \ldots, r_t') = \calA(v', h)$. In other words, if the relative number of votes of a party increases {\em ceteris paribus}, then this party cannot receive less seats. Conditions (i) and (ii) are satisfied in particular if $v$ is obtained from $v'$ by adding more votes for $P_i$.

\item[House monotonicity.] The rule $\calA$ satisfies house monotonicity if for each two numbers of available seats, $h$ and $h'$, with $h' > h$, and for each party $P_i$ it holds that $r_i' \geq r_i$, where $(r_1, \ldots, r_t) = \calA(v, h)$ and $(r_1', \ldots, r_t') = \calA(v, h')$. 
\end{description}

The property that we call party population monotonicity is sometimes called population monotonicity (for instance, this is often the case in the literature on fair allocation, and sometimes in the literature on apportionment~\cite{RePEc:hal:pseose:halshs-00879779}). However, most commonly in the context of apportionment, the term ``population monotonicity'' is used to refer to a stronger property, which covers cases when voters migrate between parties~\cite{balinski1979criteria} (intuitively, party population monotonicity describes only cases when the population of one party grows while the populations of others remain unchanged). In particular, it is known that only divisor methods satisfy population monotonicity~\cite{Balinski01}.
Party population monotonicity is more interesting for our study, since we will show that it is satisfied by the Hamilton method, and so this property will be useful in understanding the relation between the Hamilton method and its multi-attribute counterpart that we introduce in this paper.  

One assumption that is often implicitly made in the analysis of the apportionment methods is that each party has at least $h$ members, i.e., that there will always be enough candidates in each party to be given the allocated seats. This assumption will be very relevant in our further discussion. We will refer to it as to the \emph{full supply property}.  
It is commonly known that under full supply property the Hamilton method satisfies non-reversal and respect of quota, and that it fails house monotonicity (this failure of house monotonicity is better known under the name {\em Alabama paradox}). It is also known that the Hamilton method fails population monotonicity, and we will show that it satisfies its weaker variant---the party population monotonicity. On the other hand, the d'Hondt method satisfies all four properties except the respect of quota.

\begin{proposition}\label{prop:hamilton_party_population}
Under full supply property the Hamilton method satisfies party population monotonicity.
\end{proposition}
\noindent The proof of \autoref{prop:hamilton_party_population}, as all proofs omitted from the main text, is relegated to the appendix.\medskip

We note that there are also other properties of the apportionment methods considered in the literature, such as consistency, or the properties that deal with strategyproofness issues, such as resistance to party merging or to party splitting. We selected the above four properties for our analysis as the most basic ones, and perhaps the most often referred to in the literature. Moreover, as we shall soon see, they are relevant for our multi-attribute generalisation of apportionment, which does not seem to be the case for other properties listed above.

\section{The Multi-Attribute Model}\label{model}

In this section we give a formal description of our model and discuss its specific elements. We explain that our model can be viewed as a generalisation of the apportionment setting to the case of multiple attributes and we discuss how the properties of the apportionment methods from \autoref{sec:apportionmentProperties} can be formulated in such a generalised model.

\subsection{The Formal Setting}

Let $X = \{X_1, \ldots, X_p\}$ be a set of $p$ {\em attributes}, each with a finite domain $D_i = \{x_i^1, \ldots, x_i^{q_i}\}$. We say that $X_i$ is binary if $|D_i| = 2$. We set $D = D_1 \times \ldots \times D_p$, and let $C = \{c_1, \ldots, c_m\}$ be a set of \emph{candidates} (also referred to as {\em items}); $C$ is also referred to as the {\em candidate database}. Each candidate $c_i$ is represented as a vector of attribute values $(X_1(c_i) , \ldots, X_p(c_i)) \in D$.\footnote{By writing $X_j(c_i)$, we slightly abuse notation, that is, we consider $X_j$ both as an attribute name and as a function that maps a candidate to an attribute value, yet this will not lead to any ambiguity.}

For each $i \in [p]$, by $\pi_i$ we denote a \emph{target distribution} $\pi_i = (\pi_i^1, \ldots, \pi_i^{q_i})$ with $\sum_{j = 1}^{q_i} \pi_i^j = 1$. We set $\pi = (\pi_1, \ldots, \pi_p)$.
Typically, $n$ voters have cast a ballot expressing their preferred value on every attribute $X_i$, and $\pi_i^j$ is the fraction of voters who have $x_i^j$ as their preferred value for $X_i$, but the results presented in the paper are independent from where the values $\pi_i^j$ come from (see the discussion in the introduction). 

A \emph{multi-attribute committee selection rule}\footnote{We will stick to the terminology ``committee'' although the meaning of subsets of candidates has sometimes nothing to do with electing a committee.} is a function $\calR$ that for each database of candidates $C$, each vector of target distributions $\pi$ and each committee size $k \in [m]$, select a set of $k$ candidates $\calR(C,\pi,k)$ from $C$.\footnote{Observe that the outputs of the committee selection rules for the single-attribute and the multi-attribute cases are different. In the single-attribute model the rule returns a vector while in the multi-attribute one, it returns a set of $k$ candidates. This is because in the multi-attribute case it is often not reasonable to assume the full supply property, hence the exact structure of the candidate database is important---in particular, for our problem, it will be important which candidates exist in the database (and then which ones will be selected from it).} 
Again, we focus on the {\em resolute} version of such rules, using a tie-breaking mechanism whenever necessary.

Intuitively, a good multi-attribute committee selection rule should select such candidates that the distribution of attribute values in the selected set is as close as possible to $\pi$.
Let $S_k(C)$ denote the set of all subsets of $C$ of cardinality $k$. Given $A \in S_k(C)$, the {\em representation vector} for $A$ is defined as
$r(A) = \left( r_1(A), \ldots, r_p(A) \right)$, where $r_i(A) = (r_i^j(A) \mid j \in [q_i])$ for each $i \in [p]$, and $r_i^j(A) = \frac{|\{c \in A \mid X_i(c) = x_i^j\}|}{k}$.

\begin{example}
Consider the example from the introduction. In this example there are $p = 4$ four attributes, $X = \{\text{Gender}, \text{Group}, \text{Age}, \text{Affiliation}\}$. There are two possible values for the attribute ``Gender'', thus $D_1 = \{F, M\}$ ($x_1^1 = F$ and $x_1^2 = M$). Similarly, $D_2 = \{A, B, C\}$, with $x_2^1 = A$, $x_2^2 = B$, and $x_2^3 = C$, etc. For example, Ann can be represented in our model as a tuple $(F, A, J, L)$. The target distributions used in the introduction can be formulated in our model as:
\begin{align*}
\pi_1^1 = 0.5 & & \pi_2^1 = 0.55  & & \pi_3^1 = 0.3 & & \pi_4^1 = 0.3 \\
\pi_1^2 = 0.5 & & \pi_2^2 = 0.25  & & \pi_3^2 = 0.7 & & \pi_4^2 = 0.7 \\
              & & \pi_2^3 = 0.2   & &               & & 
\end{align*}
Consider a three-element committee $A = \{\text{Ann}, \text{Donna}, \text{George}\}$. For this committee, the values in appropriate representation vectors are the following:
\begin{align*}
r_1^1(A) = \nicefrac{2}{3} & & r_2^1(A) = \nicefrac{2}{3}  & & r_3^1(A) = \nicefrac{1}{3} & & r_4^1(A) = \nicefrac{1}{3} \\
r_1^2(A) = \nicefrac{1}{3} & & r_2^2(A) = \nicefrac{1}{3}  & & r_3^2(A) = \nicefrac{2}{3} & & r_4^2(A) = \nicefrac{2}{3} \\
                           & & r_2^3(A) = 0                & &                            & & 
\end{align*} \qed
\end{example}

\noindent
The following definition formalises our intuition regarding which committees are considered ideal. 

\begin{definition}\label{def:perfect_committee}
A committee 
$A \in S_k(C)$ is {\em perfect} for $\pi$ if $r_i(A) = \pi_i$ for all $i$.
\end{definition}

Thus, a perfect committee matches exactly the target distributions. Clearly, there is no perfect committee if for some $i, j$, $\pi_i^j$ is not an integer multiple of  $\nicefrac{1}{k}$. In some of our results we will focus on target distributions such that for each $i, j$ the value $k\pi_i^j$ is an integer.
We will refer to such target distributions as to {\em natural} distributions. Further, we see that the number of possible combinations of the attribute values grows exponentially with the number of nontrivial attributes (attributes which have at least two values). Consequently, even for natural distributions, finding a perfect committee cannot always be possible, simply because there are not sufficiently many appropriate candidates in the database. This observation suggests that the multi-attribute analog of the full supply property might be harder to satisfy in the multi-attribute setting (especially when the number of attributes is large).
Below, we generalise the definition of the full supply property to the case of multiple attributes.  

\begin{definition}
A candidate database $C$ satisfies the {\em full supply property} with respect to $k$ if for any $\vec x \in D$ there are at least $k$ candidates in $C$ who have the value of each attribute $X_i$ equal to $\vec x[i]$. 
\end{definition}

An alternative interpretation of the full supply property is that it is always possible to {\em create} any number of candidates corresponding to a specific vector of attribute values.
 
While in some cases, specifically when the number of attributes is very low and the number of candidates is very large, it is reasonable to expect that the database will satisfy the full supply property, it is a much less realistic assumption when the number of attributes is large and/or when the database is small. In this paper we show how to deal with such cases, and in particular, how to extend two standard methods of apportionment, the Hamilton method and the d'Hondt method, to the case of multiple attributes and to the case where the full supply property is violated. Nevertheless, the analysis of the full supply property for multiple attributes can also give us some interesting insights into the nature of the analysed multi-attribute committee selection rules, and in particular, it can allow us to view some of these rules as extensions of the classic methods of apportionment.

We can observe that there exists a straightforward polynomial-time algorithm for checking if the candidate database satisfies the full supply property. Indeed, first we need to check if the size of the database is at least equal to $k|D|$, and if this is the case, we should additionally check if for each $\vec x \in D$ there exists at least $k$ appropriate candidates in the database.

\subsection{Properties of Multi-Attribute Committee Selection Rules}

We now generalise the properties discussed in \autoref{sec:apportionmentProperties} to multi-attribute committee selection rules (which we simply refer to as ``rules'').

\begin{description}
\item[Non-reversal.] Intuitively, non-reversal says that if one value $x$ of a certain attribute has a target value higher than that of another value $y$, then $x$ should be represented in the resulting committee at least as well as $y$ (in short, values that deserve more get more). Formally, a rule $\calR$ satisfies non-reversal if for each triple $(C, \pi, k)$, if $\calR(C, \pi, k) = A$, then for all $i \in [p]$ and $j,j' \in [q_i]$,
$\pi_i^j > \pi_i^{j'}$ implies $r_i^j(A) \geq r_i^{j'}(A)$.

\item[Respect of quota.] A rule respects quota if its results match target distributions ``almost''  exactly, i.e., if they are allowed not to match them exactly, only because of the rounding issues. Formally, a rule $\calR$ respects quota if for each triple $(C, \pi, k)$, if $\calR(C,\pi,k) = A$, then for all $i \in [p]$ and $j \in [q_i]$, either $r_i^j(A) = \lfloor k \pi_i^j \rfloor$ or $r_i^j(A) = \lceil k \pi_i^j \rceil$.

\item[Value monotonicity (with respect to attribute ${\bf X_i}$).] Informally, value monotonicity says that if we increase the demand for a certain value $x$ of a certain attribute $X_i$ while not changing other demands, then in the new committee $x$ should be at least as well represented as in the old one. Formally, consider a candidate database $C$, an integer $k$, and two vectors of target distributions $\pi$ and $\rho$, such that there exist $i, j$ with:
\begin{inparaenum}[(i)]
\item $\pi_i^j > \rho_i^j$, 
\item $\nicefrac{\pi_i^{j''}}{\pi_i^{j'}} = \nicefrac{\rho_i^{j''}}{\rho_i^{j'}}$ for all $j', j'' \neq j$, and 
\item $\rho_{i'}^j = \pi_{i'}^j$ for all $i' \neq i$ and all $j \in [q_{i'}]$. 
\end{inparaenum}
$\calR$ satisfies value monotonicity with respect to attribute $X_i$ if for each such $C$, $k$, $\pi$ and $\rho$, if $\calR(C,k,\pi) = A$ and $\calR(C,k,\rho) = B$, then
$r_i^j (A)\geq r_i^j(B)$.

\item[House monotonicity.] Informally, house monotonicity says that if we increase the number of available seats, then in the new committee each value of each attribute will be at least as well represented as in the old one. Formally, a rule $\calR$ satisfies house monotonicity if for any candidate database $C$, target distribution $\pi$, and two integers $k$ and $k'$ with $k' > k$, if $\calR(C,\pi,k) = A$ and $\calR(C,\pi,k') = B$ then $r_i^j (B) \geq r_i^j(A)$ holds for all $i, j$.
\end{description}

\noindent
Clearly, these four definitions generalise the classical definitions of the properties when there is a single attribute (in particular, value monotonicity generalises party population monotonicity). 
Even though there could be other generalisations, those are arguably natural ones. 
In what follows, we will use these properties to argue that a certain class of multi-attribute committee selection rules can be viewed as extensions of the appropriate methods of apportionment.

\section{Multi-Attribute Extensions of Methods of Apportionment}\label{sec:multiAttributeRules}

As we argued in the previous section, finding perfect committees in many cases might not be feasible, either because the target distributions are not natural, or because the candidate database does not satisfy the full supply property.  These two observations lead us to define two metrics measuring how well a committee fits a target distribution. These two metrics induce two different methods of finding committees, which can be viewed as extensions of the Hamilton rule and of the d'Hondt rule to the multi-attribute domains. Other metrics will be briefly considered in \autoref{discussion-metric}.

\subsection{Multi-Attribute Hamilton Rule}\label{sec:multiAttributeHamiltonRules}

We start by defining what be believe to be the most natural metric measuring how close a given committee is to a target distribution and by arguing that such metric induces a rule which can be viewed as an extension of the Hamilton method: the $L^1$ metric. 

\begin{definition}[Multi-Attribute Hamilton Rule]\label{def:multiAttrHamilton}
The multi-attribute Hamilton rule is the function $\calR_{\mathrm{H}}$ that given a candidate database $C$, a vector of target distributions $\pi$, and an integer $k$, finds a committee $A \in S_k(C)$ minimising $\sum_{i, j} |r_i^j(A) -  \pi_i^j|$.
\end{definition}

\noindent
In other words, Multi-Attribute Hamilton Rule minimises the total variation distance between $\nicefrac{\pi}{p}$ and $\nicefrac{r}{p}$ viewed as probability distributions (we normalise $\pi$ and $r$, multiplying them by $\nicefrac{1}{p}$, so that their sums are equal to $1$, and so that they could be viewed as probability distributions).
The above definition is illustrated by the following example.

\begin{example}\label{running}
For the example from the introduction, we have $X$ = \{Gender, Group, Age, Affiliation\}, $D = \{F,M\} \times \{A,B,C\} \times \{J,S\} \times \{L,E\}$, and $X_1(\mathrm{Ann}) = F$, $X_1(\mathrm{Bob}) = M$ etc. Further, we have $\pi_1 = (0.5, 0.5)$, $\pi_2 = (0.55, 0.25, 0.2)$, $\pi_3 = (0.3, 0.7)$, and $\pi_4 = (0.3, 0.7)$.
For $k=4$, there are eight different committees which minimise our expression---let us show the calculation for one of them: $A = \{\mathrm{Ernest, George, Helena, Laura}\}$. We have $\sum_{i, j} |r_i^j(A) -  \pi_i^j| = 0 + 0 + 0.05 + 0 + 0.05 + 0.05 + 0.05 + 0.2 + 0.2 = 0.6$. \qed
\end{example}

Now, let us argue that \autoref{def:multiAttrHamilton} can be viewed as an extension of the Hamilton method of apportionment. We start by considering the case of a single-attribute ($p = 1$). Without loss of generality, let us assume that the single attribute is a party affiliation. Moreover, let us for a moment assume the full supply property, i.e., that for each value $x_1^j$ there are at least $k$ candidates with value $x_1^j$ (this is typically the case in party-list elections; in fact we need to use this assumption only to ensure that the Hamilton method is well defined). After defining $\pi_1^i = \frac{v_i}{v_{+}}$ for all $i$, we obtain the following result:

\begin{proposition}\label{prop:singleAttributeProperties}
When $p = 1$ and assuming there are at least $k$ candidates for each value of the unique attribute, then $\calR_{\mathrm{H}}$ coincides with the Hamilton apportionment rule.
\end{proposition}

\autoref{def:multiAttrHamilton} is inspired by the idea of {\em distance rationalisation of voting rules}~ (see \cite{elk-fal-sli:j:dr} for one of the most recent works on the topic). In distance rationalisation, one first defines the outcome of a voting rule for elections where there exist an obvious, non-controversial winning candidate; such elections are called \emph{consensus elections}. Second, in order to define an outcome of a voting rule for an arbitrary election $E$ we determine the closest consensus election $E'$ according to some distance (for instance, one can use the swap distance to measure the closeness between preferences of voters expressed as rankings over candidates), and we set the outcome of the rule for $E$ to the same set of winners as the outcome for $E'$. When viewed from this perspective, \autoref{prop:singleAttributeProperties} can be viewed as giving a distance rationalisation of the Hamilton rule---the consensus apportionments are those where there exist perfect committees, and the distance is the Manhattan (or $L^1$) distance.

Under the full supply assumption, a very similar result to \autoref{prop:singleAttributeProperties} holds for multiple attributes.

\begin{proposition}\label{prop:fs}
Consider a candidate database that satisfies the full supply property. For any attribute $X_i$, any committee $A$ that minimises $\sum_{i, j} |r_i^j(A) -  \pi_i^j|$ is a Hamilton committee for the single-attribute problem $(\{X_i\},D^{\downarrow X_i}, \pi_i, k)$, where  $D^{\downarrow X_i}$ is the projection of $D$ on $\{X_i\}$.
\end{proposition}

\begin{proof}
Let us fix an attribute $X_i$ and towards a contradiction, let us assume that there exists a committee $A$ that minimises $\sum_{i, j} |r_i^j(A) -  \pi_i^j|$ and that is not a Hamilton committee for the single-attribute problem $(\{X_i\},D^{\downarrow X_i}, \pi_i, k)$. By \autoref{prop:singleAttributeProperties} we infer that there exists another committee $B$ such that \mbox{$\sum_{j} |r_i^j(B) -  \pi_i^j| < \sum_{j} |r_i^j(A) -  \pi_i^j|$}. We will show that it is possible to construct a committee $D$ from $A$ and $B$ such that $\sum_{i, j} |r_i^j(D) -  \pi_i^j| < \sum_{i, j} |r_i^j(A) -  \pi_i^j|$. This will contradict the assumption that $A$ that minimises $\sum_{i, j} |r_i^j(A) -  \pi_i^j|$.

We construct $D$ as follows. We sort $A$ and $B$ in some fixed arbitrary way; let $A = \{a_1, \ldots, a_k\}$ and $B = \{b_1, \ldots, b_k\}$. For each $i\in [k]$ we take $a_i \in A$ and $b_i \in B$, and define $d_i$ as a candidate which has the value of the $i$-th attribute the same as $b_i$ and the value of all other attributes the same as $a_i$. We add $d_i$ to $D$; {\em the full supply assumption guarantees that it will always be possible to find such a candidate}. We have that:
\begin{align*}
\sum_{j} |r_i^j(D) -  \pi_i^j| = \sum_{j} |r_i^j(B) -  \pi_i^j| < \sum_{j} |r_i^j(A) -  \pi_i^j|
\end{align*}
and that for each $\ell \neq i$ it holds that:
\begin{align*}
\sum_{j} |r_{\ell}^j(D) -  \pi_{\ell}^j| = \sum_{j} |r_{\ell}^j(A) -  \pi_{\ell}^j| \text{.}
\end{align*}
Consequently, $\sum_{i, j} |r_i^j(D) -  \pi_i^j| < \sum_{i, j} |r_i^j(A) -  \pi_i^j|$, which gives a contradiction and completes the proof.
\end{proof}

Note that the construction given in the proof gives us a practical way of constructing an optimal committee under the full supply assumption.

Further, below we also show that under the full supply assumption, the multi-attribute Hamilton rule satisfies the same from the four (multi-attribute variants) of the properties considered in \autoref{sec:apportionmentProperties} as the Hamilton method of apportionment. We start by noticing that if a property fails to be satisfied in the single-attribute case, {\em a fortiori} it is not satisfied in the multi-attribute case.
As a consequence, house monotonicity is not satisfied, even under the full supply assumption. 

\begin{proposition}\label{prop:multiAttributeProperties}
Under the full supply assumption, non-reversal, respect of quota, and value monotonicity with respect to every attribute are all satisfied by the multi-attribute Hamilton rule.
In the general case, non-reversal, and respect of quota are not satisfied. If $X_i$ is a binary variable, then value monotonicity with respect to $X_i$ is satisfied; however it is not satisfied in the general case.
\end{proposition}

Importantly, if a perfect committee $A$ exists for $C$, $\pi$ and $k$, then $\calR_{\mathrm{H}} (C,\pi,k) = A$.

We close this subsection by a short discussion on the metric that is minimised in the definition of the multi-attribute Hamilton rule. It is perfectly reasonable to consider other metrics such as $\max_{i, j} |r_i^j(A) -  \pi_i^j|$ (max-max) or $\sum_{i}\max_{j} |r_i^j(A) -  \pi_i^j|$ (sum-max). In this paper we focus on $\sum_{i, j} |r_i^j(A) -  \pi_i^j|$ 
because we believe this is the most natural choice. For a discussion on other metrics we refer the reader to the conference version of this paper~\cite{conf/aaai/LangS16}. (Note that \autoref{prop:fs} does not hold with the max-max metric.) 

Finally, note that similar extensions to other largest remainder methods can be defined in the same way, after changing the value of the quota.  

\subsection{Multi-Attribute d'Hondt Rule}\label{sec:multiAttributeDHondtRules}

In this subsection we extend the idea implemented in the d'Hondt apportionment method to the multi-attribute setting. This generalisation would work for any other divisor method; for the sake of brevity, and also because the d'Hondt apportionment method is the most often used divisor method, throughout the rest of the paper we focus on this specific method. We first observe that the result of the d'Hondt apportionment can be formulated equivalently, as the solution to an optimisation problem. 

\begin{lemma}\label{lem:dHondtAlternativeFormulation}
Consider the apportionment problem and an allocation $r = (r_1, \ldots, r_t)$. If for each $i,j \in [t]$ it holds that $\nicefrac{v_i}{r_i+1} \leq \nicefrac{v_j}{r_j}$, then $r$ is a d'Hondt apportionment.
\end{lemma}
\begin{proof}
For the sake of contradiction let us assume that $r$ is not a d'Hondt apportionment. Let us run the d'Hondt method on $v$ (breaking ties arbitrarily), outputting $r^*$. Let $u$ be the last step such that $s_u(i) \leq r_i$ for all $i$: for some $i$ we have $s_{u+1}(i) = r_i + 1$ (such a step exists because $r^* \neq r$). 
By the pigeonhole principle, for some $j$ we have  $s_{u+1}(j) = s_{u}(j) < r_j$. 
By definition of the d'Hondt procedure at step $u$,  $ \frac{v_j}{s_u(j) + 1} \leq  \frac{v_i}{s_u(i)+1} = \frac{v_i}{r_i+1}$.
Since $s_u(j) < r_j$ we have $\frac{v_i}{r_i+1} \geq \frac{v_j}{s_u(j) + 1} > \frac{v_j}{r_j}$, which contradicts the condition in the statement of the lemma. 
\end{proof}

This leads us to the following equivalent formulation of the d'Hondt rule. The following proposition states an analogous result for the d'Hondt method as \autoref{prop:singleAttributeProperties} for the Hamilton rule. 

\begin{proposition}\label{prop:dhondt_alternative_definition}
In the classical apportionment setting, an allocation $(r_1, \ldots, r_t)$ maximises the value of the expression $\sum_{i \in [t]} \sum_{j \in [r_i]} \nicefrac{v_i}{j}$ if and only if it is one of the outcomes of the irresolute version of the d'Hondt apportionment rule.
\end{proposition}

\begin{proof}
Define $S(r,v) = \sum_{i \in [t]} \sum_{j \in [r_i]} \nicefrac{v_i}{j}$. Also, given an apportionment $r$ and two attribute values (parties) $i,j$ such that $r_i \neq k$ and $r_j \neq 0$, let $r[i^+ j^-]$ be the apportionment obtained from $i$ by giving one more seat to $i$ and one less to $j$ from $r$. 

First, we show that if allocation $r = (r_1, \ldots, r_t)$ maximises $S(r,v)$, then it is a d'Hondt apportionment for $v$. 
For the sake of contradiction let us assume that $r$ is not a d'Hondt apportionment. By \autoref{lem:dHondtAlternativeFormulation} we know that in such a case there exist two parties $i,j$ such that 
$\nicefrac{v_i}{r_i+1} > \nicefrac{v_j}{r_j}$. Let $r' = r[i^+ j^-]$. We have $S(r',v) =S(r,v) + \nicefrac{v_i}{r_i+1} - \nicefrac{v_j}{r_j} > S(r,v)$, therefore $r$ does not maximise $S(r,v)$.

Next, we prove that if $r$ is a d'Hondt apportionment for $v$, then it maximises $S(r,v)$. For the sake of contradiction let us assume that there exists $r'$ such that $S(r',v) > S(r,v)$. Let $r^{(0)} = r$. We define the following sequences of apportionments: for an integer $s \geq 0$, if $r^{(s)} \neq r'$ then, since $\sum_{i \in [t]} r^{(s)}_i = \sum_{i \in [t]} r_i' = k$,  there must exist two indices $i$ and $j$ such that $r_i^{(s)}  > r_i'$ and $r_j^{(s)}  < r_j'$. Let $r^{(s+1)} = r^{(s)}[j^+ i^-]$. We have $S(r^{(s+1)},v) - S(r^{(s)},v) = \nicefrac{v_j}{r_j^{(s)}+1} - \nicefrac{v_j}{r_i^{(s)}}$. Consider the step in the execution of the d'Hondt method when the $r_i$-th seat has been allocated to party $P_i$. At this step, party $P_j$ is allocated $x$ seats with $x \leq r_j \leq r_j'-1$. Since the d'Hondt method allocated the seat to party $P_i$ instead of $P_j$ it must hold that $\nicefrac{v_i}{r_i^{(s)}} \geq \nicefrac{v_i}{r_i} \geq \nicefrac{v_j}{x+1} \geq \nicefrac{v_j}{r_j+1}\geq \nicefrac{v_j}{r_j^{(s)}+1}$. Thus, each transformation does not increase the value of the expression. Yet, after a number of such transformations we reach $(r_1', \ldots, r_t')$ which has a higher value of the expression than $(r_1, \ldots, r_t)$. This gives a contradiction and completes the proof.

\end{proof}

Observe that $\sum_{j \in [r_i]} \nicefrac{v_i}{j} = v_i \harmonic(r_i)$, where $\harmonic(n) = \sum_{i=1}^n \nicefrac{1}{i}$ denotes the $n$-th harmonic number. 
\autoref{prop:dhondt_alternative_definition} leads us to the following extension of the d'Hondt method to multi-attribute scenarios.

\begin{definition}[Multi-Attribute d'Hondt Rule]\label{def:multiAttrDHondt}
The multi-attribute d'Hondt rule is the function $\calR_{\mathrm{dHondt}}$ that given a candidate database $C$, a vector of target distributions $\pi$, and an integer $k$, finds a committee $A \in S_k(C)$ maximising $\sum_{i, j} \pi_i^j\harmonic(r_i^j(A) \cdot k)$.
\end{definition}

Since for each $x \in \naturals$ we have that $\log(x + 1) \leq \harmonic(x) \leq \log(x + 1) + 1$, the maximisation of $\sum_{i, j} \pi_i^j\harmonic(r_i^j(A) \cdot k)$ is intuitively a very close objective to the maximisation of \mbox{$\sum_{i, j} \pi_i^j\log(r_i^j(A) \cdot k)$}, which is equivalent to the maximisation of \mbox{$\sum_{i, j} \pi_i^j\log\left(\nicefrac{r_i^j(A)}{\pi_i^j}\right)$}, and so, to the minimisation of \mbox{$\sum_{i, j} \pi_i^j\log\left(\nicefrac{\pi_i^j}{r_i^j(A)}\right)$}, which is the Kullback--Leibler divergence from $\nicefrac{r}{p}$ to $\nicefrac{\pi}{p}$ viewed as probability distributions.\footnote{Thanks to one of the anonymous reviewers for this observation.} 

\begin{example}
Consider again our running example. For $k = 4$ there are two optimal committees $\{\mathrm{Bob, Donna, Ernest, Helena}\}$ and $\{\mathrm{Bob, Charlie, Donna, Helena}\}$. The value of the optimised function for the first committee can be computed as $0.5 \cdot \harmonic(2) + 0.5 \cdot \harmonic(2) + 0.55 \cdot \harmonic(2) + 0.25 \cdot \harmonic(2) + 0.3 \cdot \harmonic(1) + 0.7 \cdot \harmonic(3) + 0.3 \cdot \harmonic(1) + 0.7 \cdot \harmonic(3) = 0.6 \cdot \harmonic(1) + 1.8 \cdot \harmonic(2) + 1.4 \cdot \harmonic(3) \approx 5.866$. \qed
\end{example}

We can formulate a result for the multi-attribute d'Hondt rule that is analogous to \autoref{prop:fs} for the multi-attribute Hamilton rule.

\begin{proposition}\label{prop:fsdhondt}
Consider a candidate database that satisfies the full supply property. For any attribute $X_i$, any committee $A$ that maximises $\sum_{i, j} \pi_i^j\harmonic(r_i^j(A) \cdot k)$ is a d'Hondt committee for the single-attribute problem $(\{X_i\},D^{\downarrow X_i}, \pi_i, k)$, where  $D^{\downarrow X_i}$ is the projection of $D$ on $\{X_i\}$.
\end{proposition}

Let us now examine properties of the multi-attribute d'Hondt method. It is known that for a single-attribute case the d'Hondt method satisfies non-reversal, house monotonicity, and party population monotonicity, yet it does not respect quota\footnote{It is known that the d'Hondt method satisfies a weaker form of respect of quota---it respects lower quota, i.e., for each party $P_i$ it holds that $r_i \geq \lfloor \nicefrac{v_ih}{v_+} \rfloor$.}. Consequently, respect of quota is not satisfied by the multi-attribute d'Hondt method even under the full supply assumption. 

\begin{proposition}\label{prop:multiAttributePropertiesdHondt}
Under the full supply assumption, non-reversal, house monotonicity, and value monotonicity with respect to every attribute are all satisfied by the multi-attribute d'Hondt method.
In the general case, non-reversal and house monotonicity are not satisfied. If $X_i$ is a binary variable, then value monotonicity with respect to $X_i$ is satisfied;
however it is not satisfied in the general case.
\end{proposition}

Finally, let us observe that if a perfect committee $A$ exists for $C$, $\pi$ and $k$, then $\calR_{\mathrm{dHondt}}(C,\pi,k) = A$. This follows from \autoref{prop:dhondt_alternative_definition} and from the fact that in the single-attribute case a committee that exactly matches the target distributions is always selected by the d'Hondt method. This property, which, as we have seen, also holds for $\calR_{\mathrm{H}}$, will be useful in our further discussion on computational properties of our multi-attribute rules.

\section{Computing Multi-Attribute Rules}\label{computation}

Now, we are ready to formally define the main computational problems that we address in this paper.
\begin{problem}
We are given $X$,  $C$, $\pi$, and $k$. In the \textsc{OptimalHamiltonRepresentation} we look for a committee $A \in S_k(C)$ that minimises the expression $\sum_{i, j} |r_i^j(A) -  \pi_i^j|$. In the \textsc{OptimalDHondtRepresentation} problem our goal is to find a committee $A \in S_k(C)$ maximising $\sum_{i, j} \pi_i^j\harmonic(r_i^j(A) \cdot k)$.
\end{problem}

In this section we investigate the computational complexity of the problem of finding optimal committees.
We start with observing that the problem of deciding whether there is a perfect committee for a given instance is {\sf NP}-complete.

\begin{proposition}\label{prop:np_hard}
Given set of attributes $X$, a set of candidates $C$, a vector of target distributions $\pi$, an integer $k$, deciding whether there exists a perfect committee is {\sf NP}-complete. 
\end{proposition}
\begin{proof}
Membership is straightforward. Hardness follows by reduction from the {\sf NP}-complete problem {\sc exact cover with 3-sets}, or {\sc x3c} \cite{GJ79}. Let $I = \langle X, {\cal S}\rangle$ with $X = \{x_1, \ldots, x_{3k}\}$ and ${\cal S} = \{S_1, \ldots, S_n\}$ with $|S_i| = 3$ for each $i$. $I$ is a positive instance of {\sc x3c} iff there is a collection ${\cal S}' \subseteq {\cal S}$ with $|{\cal S}'| = k$ and $\cup \{ S | S \in {\cal S}' \} = X$. Define the following instance of {\sc perfect committee}: let $X_1, \ldots, X_{3k}$ be $3k$ binary attributes, and let $C$ consist of $m$ candidates $c_1, \ldots, c_m$ with $X_i(c_j) = 1$ if $x_i \in S_j$ and  $X_i(c_j) = 0$ if $x_i \notin S_j$. Finally, for each $i$, $\pi_i^0 = \frac{k-1}{k}$ and  $\pi_i^1 = \frac{1}{k}$. We want a committee of size $k$. $A = \{c_{i_1}, \ldots, c_{i_k} \}$ is perfect for $\pi$ if for each $X_i$,  there is exactly one $j \in \{1, \ldots, k\}$ such that $X_i(c_{i_j}) = 1$, which is equivalent to saying that for each $x_i$,  there is exactly one $S_j \in \{S_{i_1}, \ldots, S_{i_k}\}$ such that $x_i \in S_j$. Thus, there is a perfect committee for $\pi$ and $C$ if and only if $I$ is a positive instance.
\end{proof}

Since the multi-attribute Hamilton and d'Hondt methods always find a perfect committee if there exists one, this simple result implies that the decision problem associated with finding an optimal committee is {\sf NP}-hard. In the next subsections we will explore two natural approaches to alleviate the {\sf NP}-hardness of the problem: we will ask if the problem can be computed efficiently when certain natural parameters are small, and we will ask whether it can be well approximated. 

In this paper we mostly present computational results for binary domains. However, this assumption is not as restrictive as it may seem---every instance of the \textsc{OptimalHamiltonRepresentation} problem can be transformed to a new instance with binary domains in the following way:


\begin{itemize}
\item $X_{\mathrm{new}} = \{X_{i,j} \ | \ i \in [p], j \in [|D_i|] \}$; for each $i, j$ we set $D_{i, j} = \{0, 1\}$.
\item $C_{\mathrm{new}} = \{c'_l \ | \ l = 1, \ldots, m \}$; for each $\ell, i, j$ we have $X_{i,j}(c'_l) = 0$ iff $X_{i}(c_l) = x_{i}^j$.
\item $\pi_{\mathrm{new}} = (\pi_{i,j} \ | \  i \in [p], j \in [|D_i|])$, where for all $i = [p]$ and $j = [|D_i|]$, $\pi_{i,j}^0 = \pi_i^j$ and $\pi_{i,j}^1 = 1 - \pi_i^j$. 
\end{itemize}

The following proposition establishes the relation between the optimal committees for the original problem, and for the problem transformed to binary domains.

\begin{proposition}\label{prop:transformToBinary}
For a given committee $A$ and target distribution $\pi$, let $A_{\mathrm{new}}$ and $\pi_{\mathrm{new}}$ denote the committee and target distributions obtained as above.
The following holds:
\begin{align*}
\sum_{i, j} |r_i^j(A_{\mathrm{new}}) - \pi_i^j| = 2\sum_{i, j} |r_i^j(A) -  \pi_i^j|\text{.}
\end{align*}
\end{proposition}

\begin{proof}
\begin{align*}
\sum_{i, j} |r_i^j(A) - \pi_i^j| &= \sum_{i, j}\left|\frac{|\{c \in A: X_i(c) = x_i^j\}|}{k} - \pi_i^j\right| = \sum_{i, j}\left|\frac{|\{c \in A_{\mathrm{new}}: X_{i, j}(c) = 0\}|}{k} - \pi_{i,j}^0\right| \\
                &= \frac{1}{2}\sum_{i, j}\left(\left|\frac{|\{c \in A_{\mathrm{new}}: X_{i, j}(c) = 0\}|}{k} - \pi_{i,j}^0\right| + \left|\frac{|\{c \in A_{\mathrm{new}}: X_{i, j}(c) = 1\}|}{k} - \pi_{i,j}^1\right|\right) \\
                &= \frac{1}{2}\sum_{i, j} \sum_{\ell \in \{0, 1\}}|r_{i, j}^{\ell}(A_{\mathrm{new}}) - \pi_{i,j}^{\ell}| = \frac{1}{2} \sum_{i, j} |r_i^j(A_{\mathrm{new}}) - \pi_i^j|.
\end{align*}
\end{proof}

\autoref{prop:transformToBinary} has interesting implications---first, it shows that the transformed instance
has the same perfect committees, and the same optimal Hamilton committees as the original instance; then it shows how to obtain approximation guarantees for \textsc{OptimalHamiltonRepresentation} for arbitrary domains having guarantees for the problem restricted to binary domains, which will be useful in \autoref{ComputingHamilton}.

\subsection{Parameterised Complexity}\label{fpt}

In this section, we study the parameterised complexity of the problem of finding optimal Hamilton and d'Hondt committees. We are specifically interested whether for some natural parameters there exist fixed parameter tractable (FPT) algorithms. We recall that the problem is FPT for a parameter $P$ if its each instance $I$ can be solved in time $O(f(P)\cdot\mathrm{poly}(|I|))$ for some computable function $f$.
From the point of view of parameterised complexity, FPT is seen as the
class of easy problems.
There is also a whole hierarchy of
hardness classes, $\text{FPT} \subseteq W[1] \subseteq W[2] \subseteq \cdots$ (for details, we point the reader to
appropriate overviews~\cite{CyganFKLMPPS15,dow-fel:b:parameterized,flu-gro:b:parameterized-complexity,nie:b:invitation-fpt}.

Obviously, the problem admits an FPT algorithm for the parameter $m$ (we can enumerate all $k$-element subsets of the set of candidates and select the best one). Now, we present a negative result for parameter $k$ (committee size) and a positive result for the parameter $p$ (number of attributes). 

\begin{theorem}
The problem of deciding whether there exists a perfect committee is $\wone$-hard for the parameter $k$, even for binary domains.
\end{theorem}

\begin{proof}
By reduction from the $\wone$-complete \textsc{PerfectCode} problem \cite{Cesati2002163}. Let $I$ be an instance of \textsc{PerfectCode} that consists of a graph $G = (V, E)$ and a positive integer $k$. We ask whether there exists 
$V' \subseteq V$ with $|V'| = k$ such that each vertex in 
$V$ is adjacent to exactly one vertex from $V'$ (by convention, a vertex is adjacent to itself).
From $I$ we construct the following instance $I'$ of the problem of deciding whether there exists a perfect committee. For each $v \in V$ there is a binary attribute $X_v$ and a candidate $c_v$. For each $u, v \in V$, $X_v(c_u) = 1$ if and only if $u$ and $v$ are adjacent in $G$. 
We look for a committee of size $k$. For each $v$, $\pi_v^1 = 1-\pi_v^0 = \frac{1}{k}$. It is easy to see that perfect codes in $I$ correspond to perfect committees in $I'$.
\end{proof}


\begin{figure}[!t!b]
  \centering
  \includegraphics[width=0.5\linewidth]{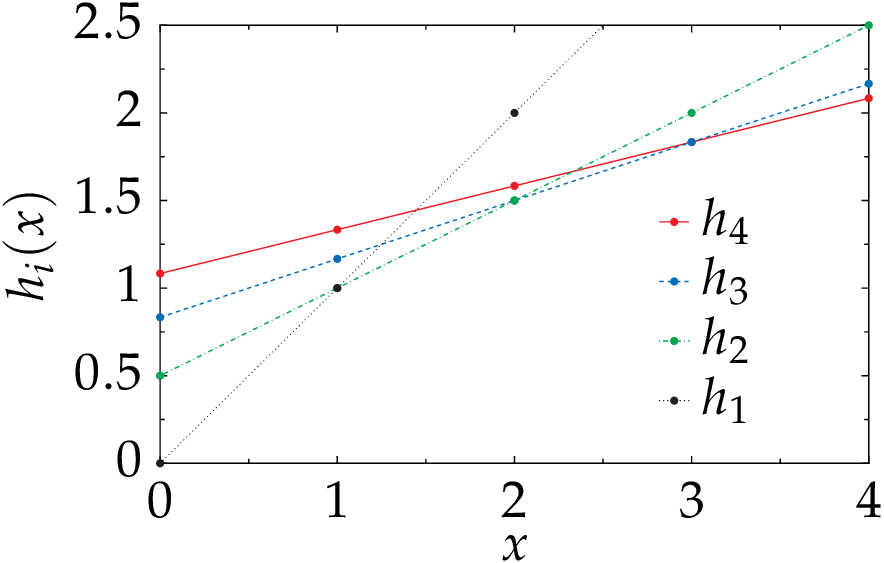}
\caption{The functions $h_1, \ldots, h_k$ used in the proof of \autoref{thm:fptParameterP} for $k = 4$. Observe that for each $x \in [k]$ the highest value $v$ such that $v \leq h_{i}(x)$ for each $i \in [k]$ is equal to $v = \harmonic(x)$.}\label{fig:decomposition_of_harmonic}
\end{figure}

\begin{theorem}\label{thm:fptParameterP}
For binary domains, there exists an $\fpt$ algorithm for \textsc{OptimalHamiltonRepresentation} and for \textsc{OptimalDHondtRepresentation} for parameter $p$. 
\end{theorem}
\begin{proof}
We will show a linear integer program for each of the two problems, \textsc{OptimalHamiltonRepresentation} and for \textsc{OptimalDHondtRepresentation}, with the number of integer variables bounded by a function of $p$. Such a linear program, by the result of Lenstra~\cite[Section~5]{Len83}, can be solved in $\fpt$ time for parameter $p$. Currently, the best known running time of algorithms solving integer linear programs is $O(n^{2.5n +o(n)}\cdot \mathrm{poly}|I|)$, where $n$ is the number of integer variables and $|I|$ is the length of encoding of the input instance~\cite{FT87ilp,Kan87}. We will start from describing the set of variables and constraints which are common for the two problems that we consider. 

Each candidate can be viewed as a vector of values indexed with the attributes; there are $2^p$ such possible vectors: $v_1, \dots, v_{2^p}$. For each $v_i$, let $a_i$ denote the number of candidates that correspond to $v_i$. For each $i \in [2^p]$ we introduce an integer variable $b_i$, which intuitively denotes the number of candidates corresponding to $v_i$ in an optimal committee. Further, for each $i \in [p]$ and each $j \in [q_i]$ we introduce a variable $r_i^j$, which in the optimal solution will be equal to $k \cdot r_{i}^j(A)$, where $A$ is the optimal committee returned by our integer linear program.
Consider the following set of linear constraints:

\newcommand{\tikzmark}[1]{\tikz[overlay,remember picture] \node (#1) {};}
\begin{figure}[H]
\begin{alignat*}{2}
                         &  \text{(a)}\colon b_i \in \integers  &\tikzmark{brace1} \ & \\
                         &  \text{(b)}\colon b_i \geq 0  &\ & \\
                         &  \text{(c)}\colon b_i \leq a_i   & \tikzmark{brace2} \ & \\
                         &  \text{(d)}\colon \sum_{i=1}^{2^p} b_i = k \ \\
                         & \text{(e)}\colon r_{i}^j = \sum_{\ell: v_\ell[i] = x_{i}^j} b_\ell & \quad\quad & i \in [p], j \in [q_i]
\end{alignat*}
\begin{tikzpicture}[overlay, remember picture]
\draw [decoration={brace,amplitude=0.4em},decorate,black]
(brace1.north east) --  (brace2.east) node [black, midway, xshift = 1.65cm, yshift = 0.0cm] {$i \in [2^p]$};
\end{tikzpicture}
\end{figure}

These constraints ensure that for a committee $A$ described by the variables $(b_i)_{i \in [2^p]}$ it holds that $r_{i}^j = k \cdot r_{i}^j(A)$, for each $i \in [p]$ and $j \in [q_i]$. Now, for \textsc{OptimalHamiltonRepresentation} we additionally introduce two real variables, $x_{i}^j$ and $y_{i}^j$, for each $i \in [p]$ and $j \in [q_i]$, and the following constraints:

\begin{figure}[H]
\begin{alignat*}{2}
                         &  \text{(f)}\colon x_i^j \geq 0  \quad\quad &\tikzmark{brace3} \ & \\
                         &  \text{(g)}\colon x_i^j \geq r_{i}^j - k \cdot \pi_{i}^j  \quad\quad  &\ & \\
                         &  \text{(h)}\colon y_i^j \geq 0   \quad\quad  & \ & \\
                         &  \text{(i)}\colon y_i^j \geq k \cdot \pi_{i}^j -r_{i}^j  \quad\quad  & \tikzmark{brace4} \ & 
\end{alignat*}
\begin{tikzpicture}[overlay, remember picture]
\draw [decoration={brace,amplitude=0.4em},decorate,black]
(brace3.north east) --  (brace4.east) node [black, midway, xshift = 1.65cm, yshift = 0.0cm] {$i \in [2^p], j\in [q_i]$};
\end{tikzpicture}
\end{figure}
These constraints ensure that for each $i \in [p]$ and $j \in [q_i]$ it holds that $k |r_{i}^j(A) - \pi_i^j| \leq x_i^j + y_i^j$. Thus, we can find an optimal Hamilton committee by minimising the objective function $\sum_{i \in [p]} \sum_{j \in [q_i]} (x_i^j + y_i^j)$ subject to constraints (a)--(i).

Finding an optimal d'Hondt committee requires an additional construction. The idea of this construction is similar to the one used by Caragiannis~et~al.~\cite{Caragiannis:2016:UFM:2940716.2940726}. 
Let us construct $k$ linear functions: $h_1, \ldots, h_k$, as follows. For each $i \in [k]$ we define $h_i$ as a linear function such that $h_i(i-1) = \harmonic(i-1)$ and $h_i(i) = \harmonic(i)$, where $\harmonic(i)$ denotes the $i$-th harmonic number. The functions $(h_i)_{i \in [k]}$ for $k = 4$ are depicted in \autoref{fig:decomposition_of_harmonic}. Now, for each $i \in [p]$ and $j \in [q_i]$ we introduce one real variable $z_{i}^j$ and the following constraints:
\begin{alignat*}{2}
                         &  \text{(j)}\colon z_i^j \leq h_{\ell}(r_{i}^j)  \quad\quad  &\ & \ell \in [k]\\
\end{alignat*}
The highest value of $z_i^j$ which satisfies constraints $(j)$ is equal to $\harmonic(r_{i}^j)$. Thus, to find an optimal d'Hondt committee we need to maximise the expression $\sum_{i \in [p]} \sum_{j \in [q_i]} z_i^j$ subject to constraints (a)--(e) and (j).

Of course, when there is no candidate corresponding to a given vector of values of the attributes $v_i$, then we can skip the respective variable $b_i$. This can make our ILPs more practical when the size of the candidate database is small.
\end{proof}

\begin{example}
Let $p = 2$, $k = 5$, and let the candidate database $C$ consists of 4 candidates with value vector $v_1 = (0,0)$,   2 with value vector $v_2 = (1,0)$, 2 candidates with value vector $v_3 = (0,1)$ and 2 candidates with value vector $v_4 = (1,1)$.  Let $\pi = ((0.2, 0.8), (0.6, 0.4))$.   The constraints (a)--(e) of the integer linear program are: 
{\upshape
  \begin{align*}
    & \text{(a)}\colon b_i \in \integers                                                                        \ & 1 \leq i \leq 4 \\
    & \text{(b)}\colon b_i \geq 0                                                                        \ & 1 \leq i \leq 4 \\
    & \text{(c)}\colon b_1 \leq 4; \ b_2 \leq 2; \ b_3 \leq 2; \ b_4 \leq 2  \\
    & \text{(d)}\colon b_1 + b_2 + b_3 + b_4 = 5                                                          \ \\
    & \text{(e)}\colon r_{1}^0 = b_1 + b_3; \ r_{1}^1 = b_2 + b_4; \ r_{2}^0 = b_1 + b_2; \  r_{2}^1 = b_3 + b_4;
  \end{align*}
}
and a solution is $(b_1 = 1, \ b_2 = 2, \ b_3 = 0, \ b_4 = 2)$: a perfect committee is obtained by taking one candidate with value vector $(0,0)$, two candidates with value vector $(1,0)$, and two with value vector $(1,1)$. Thus, this is an optimal d'Hondt and Hamilton committee. \qed
\end{example}

As a corollary of \autoref{thm:fptParameterP} we get that the problem of checking whether there exists a perfect committee is in $\fpt$ for the parameter $p$. We can see that the proof of \autoref{thm:fptParameterP} easily extends to the case where the size of each domain $D_i$ is bounded by a constant.

We conclude this section by a short discussion. Finding an optimal committee is likely to be difficult if the candidate database $C$ is large, and the number of attributes not small. Assume $|C|$ is large compared to the size of the domain $\prod_{i=1}^p |D_i|$, that each attribute value appears often enough in $C$ and that there is no strong correlation between attributes in $C$: then, the larger $|C|$, the more likely $C$ satisfies the full supply property, in which case finding an optimal committee is easy. The really difficult cases are when $|C|$ is not significantly larger than the domain, or when $C$ shows a high correlation between attributes.

We have run a set of simple experiments to better understand the limitations of the ILP-based approach presented in the proof of \autoref{thm:fptParameterP}. For several different values of the parameters $m$ (the number of candidates), $p$ (the number of binary attributes) and $k$ (the size of the committee) we run the following simulations. We selected the value of each attribute for each candidate independently, following a uniform i.i.d. distribution. For each $i \in [p]$ we set the target distribution for the $i$-th attribute to $\pi_i^0 = \pi_i^1 = \nicefrac{1}{2}$. For each combination of the parameter values $(m, p, k)$ we run 10 experiments and computed the average and the worst-case time that the appropriate ILP required to solve the respective instance. For our experiments we used the Gurobi ILP solver and a desktop machine with 4 processors Intel Core i5-4200U, 1.62GHz, 3072KB of cache. The running times of our ILPs are summarised in \autoref{tab:running_times}.    

\begin{table}[t!]
\centering
\begin{minipage}{.45\linewidth}
      \centering
      \begin{tabular}{ c | c | c | c | c }
         $m$ & $p$ & $k$ & average & maximal \\
         \hline
         $50$ & $20$ & $5$ & $0.06$s & $0.07$s \\
         \hline
         $50$ & $20$ & $10$ & $0.33$s & $1.17$s \\
         \hline
         $100$ & $50$ & $10$ & $55$s & $98$s \\
         \hline
         $100$ & $50$ & $25$ & $2.7$h & $9.5$h \\
         \hline
         $300$ & $80$ & $40$ & $>10$h & $>10$h \\
       \end{tabular}\\
       \bigskip
        a. Finding optimal Hamilton committees
\end{minipage}%
\begin{minipage}{.45\linewidth}
      \centering
      \begin{tabular}{ c | c | c | c | c }
         $m$ & $p$ & $k$ & average & maximal \\
         \hline
         $50$ & $20$ & $10$ & $0.07$s & $0.09$s \\
         \hline
         $50$ & $20$ & $10$ & $1.05$s & $1.73$s \\
         \hline
         $100$ & $50$ & $10$ & $59$s & $89$s \\
         \hline
         $100$ & $50$ & $25$ & $1.9$h & $4.3$h \\
         \hline
         $300$ & $80$ & $40$ & $>10$h & $>10$h \\
       \end{tabular}\\
       \bigskip
        b. Finding optimal d'Hondt committees
\end{minipage} 
\caption{Running times (average and maximal) of the ILP for the problem of finding optimal Hamilton and d'Hondt committees. The entry ``$>10$h'' means that none of the ten experiments finished before the deadline of 10 hours.}\label{tab:running_times}
\end{table}

We observe that for a small number of attributes, we can efficiently apply our FPT algorithms. 
Nevertheless, for large instances, with more than a hundred attributes and than a few hundreds candidates, the ILP-based approach is no longer feasible. 
Following this observation, in the next two subsection we will discuss an alternative approach, which uses the concept of approximation. This approach is suitable when the number of attributes is too large for the  ILP-based algorithms. 

One more reason for studying approximation algorithms for finding a committee is that an approximation algorithm can be viewed as a new rule , which might or might not inherit some good properties of the original rule that we aim to approximate (this view of approximation algorithms was taken first in  \cite{car-kak-kar-pro:j:dodgson-acceptable}). This new rule is not only simpler to compute but also may be easier to understand by humans. Moreover, for low-stake domains where optimality is not crucial,  it often does not matter whether we apply the initial rule of one of its approximations.

\subsection{Approximating the Multi-Attribute d'Hondt Method}\label{ComputingHondt}

Let us first consider the problem of approximating the multi-attribute d'Hondt method. We will use the following standard definition of approximation.

\begin{definition}
An algorithm $\mathscr{A}$ is an $\alpha$-approximation algorithm for \textsc{OptimalDHondtRepresentation} if for each instance $I$ of \textsc{OptimalDHondtRepresentation} it holds that 
\begin{align*}
\sum_{i, j} \pi_i^j\harmonic(r_i^j(A) \cdot k) \geq \alpha \sum_{i, j} \pi_i^j\harmonic(r_i^j(A^*) \cdot k) \text{,}
\end{align*}
where $A$ is the committee returned by $\mathscr{A}$ for $I$, and is $A^*$ an optimal committee.
\end{definition}

\SetKwInput{KwParameters}{Parameters}
\begin{figure}[thb]
\begin{algorithm}[H]
   \small
   $A \leftarrow \emptyset$\;
   \For{$i\leftarrow 1$ \KwTo $k$}{
      $c \leftarrow \argmax_{c' \in C\setminus A} \sum_{i, j} \pi_i^j\harmonic(r_i^j(A\cup\{c'\}) \cdot k)$ \;
      $A \leftarrow A \cup \{c\}$\;
   }
   \Return{$A$}\;
\end{algorithm}
\caption{Greedy approximation algorithm for the \textsc{OptimalDHondtRepresentation} problem.}\label{alg:greedy}
\end{figure}

For the \textsc{OptimalDHondtRepresentation} problem we show that a simple greedy algorithm from \autoref{alg:greedy} achieves an approximation ratio of $1 - \nicefrac{1}{e}$.

\begin{theorem}
The greedy algorithm from \autoref{alg:greedy} is a $(1 - \nicefrac{1}{e})$-approximation algorithm for \textsc{OptimalHamiltonRepresentation}.
\end{theorem}

\begin{proof}
Consider the set function that we optimise $f(A) = \sum_{i, j} \pi_i^j\harmonic(r_i^j(A) \cdot k)$. It can be expressed as a linear combination of functions $f(A) = \sum_{i, j} f_{i, j}(A)$, where $f_{i, j}(A) = \harmonic(r_i^j(A))$. We will now show that for each $i$ and $j$ the function $f_{i, j}$ is submodular. For that we need to show that for each $A, B$ with $A \subset B \subset C$, and for each $c \in C\setminus B$ it holds that:
\begin{align*}
f_{i, j}(A \cup \{c\}) - f_{i, j}(A) \geq f_{i, j}(B \cup \{c\}) - f_{i, j}(B) \textrm{.}
\end{align*} 
Now, if $X_{i}(c) \neq x_i^j$ then $r_i^j(A \cup \{c\}) = r_i^j(A)$ and $r_i^j(B \cup \{c\}) = r_i^j(B)$, thus:
\begin{align*}
f_{i, j}(A \cup \{c\}) - f_{i, j}(A) = 0 = f_{i, j}(B \cup \{c\}) - f_{i, j}(B) \textrm{.}
\end{align*} 
On the other hand, if $X_{i}(c) = x_i^j$ then $r_i^j(A \cup \{c\}) = r_i^j(A) + 1$ and $r_i^j(B \cup \{c\}) = r_i^j(B) + 1$. Since $A \subset B$, we have $r_i^j(B) \geq r_i^j(A)$ and so:
\begin{align*}
f_{i, j}(A \cup \{c\}) - f_{i, j}(A) = \frac{1}{r_i^j(A) + 1} \geq \frac{1}{r_i^j(B) + 1} = f_{i, j}(B \cup \{c\}) - f_{i, j}(B) \textrm{.}
\end{align*}
Thus, each function $f_{i, j}$ is submodular. Also, $f$ as a linear combination of submodular functions is submodular. The thesis follows from the famous result of Nemhauser~et~al.~\cite{submodular} which established the $(1 - \nicefrac{1}{e})$-approximation bound for the greedy algorithm for the problem of maximising a submodular function.
\end{proof}

\subsection{Approximating the Multi-Attribute Hamilton Rule}\label{ComputingHamilton}

Now, we move to the problem of approximating the multi-attribute Hamilton method. Before proceeding to presentation of our approximation algorithms for this problem, 
we define the notion of approximability used in our analysis. First, we observe that there is no hope for a polynomial time approximation algorithm according to the notion of multiplicative approximation, perhaps the most commonly used definition of approximation.

\begin{proposition}
Unless ${\sf P = NP}$, for each computable function $\alpha\colon \naturals \to \reals$ there exists no polynomial-time algorithm that or each instance $I$ of \textsc{OptimalHamiltonRepresentation} returns a committee $A$ such that:
\begin{align*}
\sum_{i, j} |r_i^j(A) -  \pi_i^j| \leq \alpha(|I|) \cdot \sum_{i, j} |r_i^j(A^*) -  \pi_i^j| \text{,}
\end{align*}
where $A^*$ is an optimal committee for $I$.
\end{proposition} 
\begin{proof}
For the sake of contradiction, let us assume that such a polynomial-time algorithm exists. Then for each instance $I$ for which there exists a perfect committee $A^*$, we have $\sum_{i, j} |r_i^j(A^*) -  \pi_i^j| = 0$, and thus our algorithm would need to find a committee $A$ with $\sum_{i, j} |r_i^j(A) -  \pi_i^j| = 0$. This means that we could use our algorithm to find a perfect committee, whenever such exists. Yet, by \autoref{prop:np_hard} deciding whether there exists a perfect committee is {\sf NP}-hard. 
\end{proof}

Given this strong negative result, we move to analysing the {\em additive} approximation of the problem. 

\begin{definition}\label{def:hamilton_approximation}
An algorithm $\mathscr{A}$ is an $\alpha$-additive-approximation algorithm for \textsc{OptimalHamiltonRepresentation} if for each instance $I$ of \textsc{OptimalHamiltonRepresentation} it holds that 
\begin{align*}
\big|\sum_{i, j} |r_i^j(A) -  \pi_i^j| - \sum_{i, j} |r_i^j(A^*) -  \pi_i^j|\big| \leq \alpha \text{,}
\end{align*}
where $A$ is the committee returned by 
$\mathscr{A}$ for $I$, and $A^*$ an optimal committee.
\end{definition}

Now, we are ready to show an approximation algorithm for the \textsc{OptimalHamiltonRepresentation} problem. The algorithm is given in \autoref{alg:local-search} and is parameterised by an integer value $\ell$. It starts with a random collection of $k$ samples and, in each step, it looks whether it is possible to replace some $\ell$ candidates from the current solution with some other $\ell$ candidates to obtain a better solution (if there exist many choices for replacing such $\ell$ candidates the algorithm can pick an arbitrary of them). The algorithm continues until it cannot find any pair of sets of $\ell$ candidates that would improve the current solution.

\SetKwInput{KwParameters}{Parameters}
\begin{figure}[thb]
\begin{algorithm}[H]
   \small
   \SetAlCapFnt{\small}
   \KwParameters{\\$\hspace{3pt}$ $\pi = (\pi_1, \ldots, \pi_p)$---input target distributions. \\
                 \\$\hspace{3pt}$ $\ell$---the parameter of the algorithm.}
   $A \leftarrow k$ \textrm{random candidates from} $C$\;
   \While{\textrm{there exist} $C_{\ell} \subset C$ \textrm{and} $A_{\ell} \subset A$ \textrm{such that} $|C_{\ell}| = |A_{\ell}| \leq \ell$, \textrm{and} \\
          \quad\quad\quad $\sum_{i, j} \big|r_i^j(A) -  \pi_i^j\big|  > \sum_{i, j} \big|r_i^j((A \setminus A_{\ell}) \cup C_{\ell}) -  \pi_i^j\big|$}{
      $A \leftarrow (A \setminus A_{\ell}) \cup C_{\ell}$\;
   }
   \Return{$A$}\;
\end{algorithm}
\caption{Local search approximation algorithm for the \textsc{OptimalHamiltonRepresentation} problem.}\label{alg:local-search}
\end{figure}

Let $q = \max_{i \in [p]}q_i$. For the sake of simplicity, assume that target distributions are natural (this is almost without loss of generality, as there we can always find target distributions that are close to the initial distributions; see the end of \autoref{computation}). We show that in that case,
the running time of the local search algorithm is $O(p^2 m^\ell k^{\ell+1}q\ell)$. For instance, the simplest variant of the algorithm (that is, the algorithm for $\ell = 1$) for binary domains works in time $O(m p^2 k^{2})$. Indeed, the algorithm starts with a random committee $A$; the worst case distance to the target distributions can be upper bounded by $\sum_{i, j} \big|r_i^j(A) -  \pi_i^j\big| \leq \sum_{i, j} 1 \leq p q$. In each iteration of the {\em while} loop the solution improves: since the distributions are natural, the distance must improve by at least $\nicefrac{1}{k}$. This is because for each $i \in [p], j \in [q_i]$ the values $r_i^j(A)$ and $\pi_i^j$ are integral multiples of $\nicefrac{1}{k}$, and so is the optimised value. Thus, there will be at most $pqk$ iterations of the {\em while} loop. In each iteration we check all $\ell$-element subsets of the set of candidates and compare each such a subset with all $\ell$-element subsets of the current best committee --- thus, there are at most $m^\ell k^\ell$ such comparisons. For a single comparison we need to check all the attributes of the selected candidates in order to verify if replacing the appropriate subsets gives an improvement, which results in $p \ell$ operations.  

As we show now, the approximation guarantees depend on the value of the parameter $\ell$.

\begin{theorem}\label{thm:localSearch1}
Recall that $p = |X|$. For binary domains and 
natural distributions the local search algorithm from \autoref{alg:local-search} with $\ell = 1$ is a $p$-additive-approximation algorithm for \textsc{OptimalHamiltonRepresentation}.
\end{theorem}

\begin{proof}
Let $A^{*}$ denote an optimal solution for a given instance $I$ of the \textsc{OptimalHamiltonRepresentation} problem. Let $A \in S_k(C)$ denote the set returned by the local search algorithm from \autoref{alg:local-search}. From the condition in the ``while'' loop, we know that there exist no $c \in C$ and $a \in A$ such that $\sum_{i, j} \big|r_i^j(A) -  \pi_i^j\big|  > \sum_{i, j} \big|r_i^j((A \setminus \{a\}) \cup \{c\}) -  \pi_i^j\big|$. Now, let $X_{\mathrm{ex}} \subseteq X$ denote the set of all attributes for which $A$ achieves exact match with $\pi$, that is, such that for each $X_i \in X_{\mathrm{ex}}$, we have that $r_i^1(A) = \pi_i^1$ and $r_i^2(A) = \pi_i^2$.

Let us consider the procedure consisting in taking the candidates from $A \setminus A^{*}$ and, one by one, replacing them with arbitrary candidates from $A^{*} \setminus A$. This procedure, in $|A \setminus A^{*}|$ steps, transforms $A$ into an optimal solution $A^{*}$. We now estimate the total gain $g$ induced by this procedure. For each candidate $a \in A \setminus A^{*}$, by $a' \in A^{*} \setminus A$ we denote the candidate which was taken to replace $a$ in the procedure. For each attribute $X_i \in X$ we define the gain $g_{i}(a, a')$ of replacing $a$ by $a'$ as:
{\small
\begin{align*}
g_{i}(a, a') = \sum_{j \in \{1, 2\}}\left(|r_i^j(A) - \pi_i^j| - |r_i^j(A \setminus \{a\} \cup \{a'\}) - \pi_i^j|\right) \textrm{.}
\end{align*}
}
We now extend this definition to sets of $k$ candidates:
{\small
\begin{align*}
g_{i}(B, B') = \sum_{j \in \{1, 2\}}\left(|r_i^j(A) - \pi_i^j| - |r_i^j((A \setminus B) \cup B') - \pi_i^j|\right) \textrm{.}
\end{align*}
}
If $X_i  \in X_{\mathrm{ex}}$, then $r_i(A) = \pi_i$, and so the replacement cannot improve the quality of the solution relatively to $X_i$, hence
\begin{align}\label{eq:noGainInEx}
\sum_{i \in X_{\mathrm{ex}}} g_i(A \setminus A^{*}, A^{*} \setminus A) \leq 0 \textrm{.}
\end{align}
Since the distribution is natural, we have that $g_{i}(a, a') \in \left\{-\frac{2}{k}, 0, \frac{2}{k}\right\}$. This is because replacing a single candidate in $A$ can change the value of $|\{c \in A: X_i(c) = x_i^j\}|$ by $-1$, $0$, or $1$, and so, it can change the value of $r_i^j(A)$ by $-\frac{1}{k}$, 0, or $\frac{1}{k}$. 
Moreover, for each attribute $X_i \notin X_{\mathrm{ex}}$ there are two possible cases:
\begin{enumerate}
\item $r_i^j(A) > \pi_i^j$ and each exchange 
of candidate that results in a negative gain increases $r_i^j(A)$.
\item $r_i^j(A) < \pi_i^j$ and  each exchange that results in a negative gain decreases $r_i^j(A)$.
\end{enumerate}
Intuitively, 1. and 2. mean that for attributes outside of $X_{\mathrm{ex}}$, the negative gains cumulate. Formally, for each $X \notin X_{\mathrm{ex}}$:
\begin{align}\label{eq:negativeGainsCummulate}
g_i(A \setminus A^{*}, A^{*} \setminus A) \leq  \sum_{a \in A \setminus A^{*}} g_{i}(a, a') \textrm{.}
\end{align}
From the condition in the ``while'' loop, we have that for each $a \in A \setminus A^{*}$: $\sum_i g_{i}(a, a') \leq 0$, and so:
\begin{align}\label{eq:whileLoopCondition}
\sum_{i} \sum_{a \in A \setminus A^{*}} g_{i}(a, a') \leq 0 \textrm{.}
\end{align}
We now give the following sequence of inequalities:
\begin{align*}
g &= \sum_i g_i(A \setminus A^{*}, A^{*} \setminus A) = \sum_{i \in X_{\mathrm{ex}}} g_i(A \setminus A^{*}, A^{*} \setminus A) + \sum_{i \notin X_{\mathrm{ex}}} g_i(A \setminus A^{*}, A^{*} \setminus A) \\
  &\leq \sum_{i \notin X_{\mathrm{ex}}} g_i(A \setminus A^{*}, A^{*} \setminus A) \leq \sum_{i \notin X_{\mathrm{ex}}} \sum_{a \in A \setminus A^{*}} g_{i}(a, a') \leq - \sum_{i \in X_{\mathrm{ex}}} \sum_{a \in A \setminus A^{*}} g_{i}(a, a')
\end{align*}
The last transition in the above sequence is due to \autoref{eq:whileLoopCondition}. Consequently, we get that:

\begin{align}
  g \leq \Big|\sum_{i \in X_{\mathrm{ex}}} \sum_{a \in A \setminus A^{*}} g_{i}(a, a')\Big| \leq |X_{\mathrm{ex}}| \cdot k \cdot \frac{2}{k} = 2|X_{\mathrm{ex}}|. 
  \label{eq:sequenceFromTheorem1}
\end{align}
Finally, for each attribute $X_i \notin X_{\mathrm{ex}}$ the loss relative to $X_i$, {\em i.e.}, $|r_i^0(A) - \pi^0| + |r_i^1(A) - \pi^1|$, is
at most 2. Thus, we get $g \leq 2(|X| - |X_{\mathrm{ex}}|)$, which leads to $g \leq |X|$.
\end{proof}


One way to interpret the bound $|X|$ is to observe that a solution that for half of the attributes gives exact match, and for other half is arbitrarily bad, is an $|X|$-approximate solution. We do not know whether the bound $|X|$ is reached, but below we show a lower bound of $\frac{2}{3}|X|$ on the error made by the algorithm with $\ell = 1$.

\begin{example}
Consider $3p$ binary attributes $X_1, \dots, X_{3p}$, $4\ell$ candidates $C = \{a_1, \dots, a_{2\ell}, b_1, \dots, b_{2\ell}\}$, and let $k = 2\ell$. For each $i \leq p$, we have: for $j \leq \ell, X_i(a_j) = 1$ and $X_i(b_j) = 1$; for $j > \ell, X_i(a_j) = 0$ and $X_i(b_j) = 0$. For each $i$ such that $p < i \leq 2p$ we have: for $j \leq \ell, X_i(a_j) = 1$ and $X_i(b_j) = 0$; for $j > \ell, X_i(a_j) = 0$ and $X_i(b_j) = 1$. For $i > 2p$ we have: for each $j, X_i(a_j) = 1$ and $X_i(b_j) = 0$. Finally, for $i \leq 2p$ let $\pi_{i}^0 = \pi_{i}^1 = \frac{1}{2}$, and for $i > 2p$ let $\pi_{i}^0 = 1 - \pi_{i}^1 = 1$. It can be easily checked that $B = \{b_1, \dots, b_{2\ell}\}$ is a perfect committee. Now, $A = \{a_1, \dots, a_{2\ell}\}$ is locally optimal. To check this, we consider two cases: in the first case, where ($r \leq \ell$ and $q \leq \ell$) or ($r > \ell$ and $q > \ell$), replacing $a_r$ with $b_q$ does not change the distance to the target distribution on each of the first $p$ attributes, increases the distance on each of the next $p$ attributes and decreases the distance on each of the last $p$ attributes. For the second case, where ($r \leq \ell$ and $q > \ell$) or ($r > \ell$; $q \leq \ell$), the line of reasoning is similar.
Finally, $\sum_{i, j} \big|r_i^j(A) -  \pi_i^j\big| = 2p = \frac{2}{3}|X|$. \qed
\end{example}\medskip

A better approximation bound can be obtained with $\ell = 2$, yet the analysis of this case is much more involved.

\begin{lemma}\label{lemma:localSearch2HelpingLemma}
Consider $n$ buckets $X_1, \dots, X_n$, such that in the $i$-th bucket $X_i$ there are $x_i$ white balls and $y_i$ black balls. Let $A$ denote the number of pairs of balls such that both balls in the pair belong to the same bucket and are of different color. Let us consider the procedure in which one iteratively selects a bucket and takes out two balls with different colors from the selected bucket. The procedure ends after $B$ steps, when no further steps are possible (in each bucket, either there are no balls anymore, or all balls have the same color). It holds that $A \geq \frac{B^2}{n}$.
\end{lemma}
\begin{proof}
Without loss of generality let us assume that for each $i$: $x_i \leq y_i$. Thus, $B = \sum_i x_i$ and $A = \sum_i x_iy_i \leq \sum_i x_i^2$. The inequality $\sum_i x_i^2 \geq \frac{\left(\sum_i x_i\right)^2}{n}$ follows from Jensen's inequality~\cite{Cvetkovski2012} applied to the quadratic function.
\end{proof}

\begin{lemma}\label{lemma:constraintsEstimation}
Let $x_i, y_i, A_i$, $1 \leq i \leq n$, be real values satisfying the following constraints:
\begin{enumerate}
\item $x_{i} \geq \frac{A_i}{2n-2(i-1)}$, for each $1 \leq i \leq n$,
\item $A_i \geq A_{i-1} - 2x_{i-1}$, for each $2 \leq i \leq n$,
\item $y_i \geq \frac{x_i}{2n-2(i-1)-1}$, for each $1 \leq i \leq n$.
\end{enumerate}
Then:
\begin{align*}
\sum_{i=1}^n y_i \geq \frac{|A_1|\ln n }{4n} \textrm{.}
\end{align*}
\end{lemma}
\begin{proof}
We can view the set of above inequalities 1, 2, 3 as a linear program with $(3n-1)$ variables (all $x_i$ and $y_i$ for $1 \leq i \leq n$ and  $A_i$ for $2 \leq i \leq q$; we treat $A_1$ as a constant) and $(3n-1)$ constraints. Thus, we know that $\sum_i y_i$ achieves the minimum when each from the above constraints is satisfied with equality.

We show by induction that the values $x_i = \frac{A_1}{2n}$ and $A_i = \frac{2n - 2(i-1)}{2n} A_1$ constitute the solution to the set of equalities that is derived by taking constraints 1, and 2, and treating them as equalities. We can show that by induction: First, consider the base step, i.e., the case when $i = 1$. Since constraint 2 is defined only for $i \geq 2$, we need to check only constraint 1. This constraint written in the form of equality gives us:
\begin{align*}
&x_{1} = \frac{A_1}{2n-2(i-1)} = \frac{|A_1|}{2n} \textrm{,}
\end{align*}
which proves that our hypothesis holds for $i=1$.
Next, let us assume that from the equalities 1 and 2 taken for $i < j$, it follows that $x_i = \frac{A_1}{2n}$ and $A_i = \frac{2n - 2(i-1)}{2n} A_1$, for $i < j$. We will show that from equalities 1 and 2 for $i = j$ it follows that $x_j = \frac{A_1}{2n}$ and $A_j = \frac{2n - 2(j-1)}{2n} A_1$:
\begin{align*}
&x_{j} = \frac{A_j}{2n-2(j-1)} = \frac{1}{2n-2(j-1)} \cdot \frac{2n - 2(j-1)}{2n} A_1 = \frac{|A_1|}{2n} \textrm{,}\\
&A_j = A_{j-1} - 2x_{j-1} = \frac{2n - 2((j-1)-1)}{2n} A_1 -  2\frac{|A_1|}{2n} = \frac{2n - 2(j-1)}{2n} A_1\textrm{.}
\end{align*}
From constraint 3, treated as equality, we get:
\begin{align*}
y_i = \frac{x_i}{2n-2(i-1)-1} = \frac{|A_1|}{2n(2n-2(i-1)-1)} \textrm{.}
\end{align*}
Thus, we infer that $\sum_{i=1}^n y_i$ is minimised when $y_i = \frac{|A_1|}{2n(2n-2(i-1)-1)}$. We recall that $H_n$ denotes the $n$-th harmonic number ($H_n = \sum_{i=1}^n \frac{1}{i}$), and that $\ln(n+1) < H_n \leq 1 + \ln (n)$. As a result we get:
\begin{align}
\sum_{i=1}^n y_i &\geq \frac{A_1}{2n} \sum_{i=1}^{n}\frac{1}{(2n-2(i-1)-1)} \geq \frac{A_1}{2n} \sum_{i=1}^{n}\frac{1}{2n-2(i-1)} \\
            &= \frac{A_1}{4n} \sum_{i=1}^{n}\frac{1}{(n-i+1))} = \frac{A_1}{4n} H_n \geq A_1 \frac{\ln n}{4n} \textrm{.}
\end{align}
\end{proof}

\begin{theorem}\label{thm:localSearch2}
For binary domains and natural distributions the local search algorithm from \autoref{alg:local-search} with $\ell = 2$ is a $\frac{\ln (k/2)}{2\ln (k/2) - 1}\left(|X| + \frac{6|X|}{k} \right)$-additive-approximation algorithm for \textsc{OptimalHamiltonRepresentation}.
\end{theorem}
\begin{proof}
In this proof we use similar idea to the proof of \autoref{thm:localSearch1}, but the proof is technically more involved. As before, by $A^{*}$ and $A$ we denote the optimal solution and the solution returned by the local search algorithm, respectively. Similarly to the previous proof, by $X_{\mathrm{ex}} \subset X$ we denote the set of all attributes for which $A$ achieves exact match with $\pi$, i.e.,
\begin{align*}
X_{\mathrm{ex}} = \left\{X_i \in X: r_i^1(A) = \pi_i^1 \right\} \textrm{.}
\end{align*}
We also define the set $X_{\mathrm{aex}} \subset X$ of all attributes for which $A$ achieves almost exact match with $\pi$, i.e.,
\begin{align*}
X_{\mathrm{aex}} = \left\{X_i \in X: |r_i^1(A) - \pi_i^1| \leq \frac{1}{k} \right\} \textrm{.}
\end{align*}
Let $q_f = \frac{|A \setminus A^{*}|}{2}$ and $q = \lfloor q_f \rfloor$. Let us rename the candidates from $A \setminus A^{*}$ so that $A \setminus A^{*} = \{a_1, a_2, \dots, a_{2q_f}\}$, and the candidates from $A^{*} \setminus A$, so that $A^{*} \setminus A = \{a_1', a_2', \dots, a_{2q_f}'\}$. Hereinafter, we follow a convention in which the elements from $A^{*} \setminus A$ are marked with primes. Renaming of the candidates that we described above, allows us to the define the following sequence of pairs $(a_1, a_1'), \dots, (a_{2q_f}, a_{2q_f}')$ in which each element from $A \setminus A^{*}$ is paired with (assigned to) exactly one element from $A^{*} \setminus A$.

For each pair $(a_j, a_j')$ and for each attribute $X_i$ we consider what happens if we replace $a_i$ in $A \setminus A^{*}$ with $a_i'$. One of three scenarios can happen, after such a replacement:
\begin{enumerate}
\item The value $r_i^0(A)$ can increase by $\frac{1}{k}$ (in this case $r_i^1(A)$ decreases by $\frac{1}{k}$), which we denote by \mbox{$X_{i}(a_j \leftrightarrow a_j') = 1$},
\item The value $r_i^0(A)$ can decrease by $\frac{1}{k}$ (in this case $r_i^1(A)$ increases by $\frac{1}{k}$), which we denote by \mbox{$X_{i}(a_j \leftrightarrow a_j') = -1$}, or
\item The value $r_i^0(A)$ can remain unchanged (in this case $r_i^1(A)$ also remains unchanged), which we denote by \mbox{$X_{i}(a_j \leftrightarrow a_j') = 0$}.
\end{enumerate}

We follow a procedure which, in $q$ consecutive steps, replaces pairs of candidates from $A \setminus A^{*}$, with the pairs of candidates from $A^{*} \setminus A$. A pair $(a_i, a_j)$ is always replaced with $(a_i', a_j')$. In other words, when looking for a pair from $A^{*} \setminus A$ to replace $(a_i, a_j)$ we follow the assignment rule induced by renaming, as described above. The way in which we create pairs within $A \setminus A^{*}$ for replacement (the way how $(a_i, a_j)$ is selected in each of $q$ consecutive steps) will be described later.
After this whole procedure $A$ can differ from $A^{*}$ with at most one element, hence, having distance to the optimal distribution at most equal to $|X|\frac{2}{k}$. Let us define the sequence of sets $\bar{A}_1, \bar{A}_2, \dots, \bar{A}_{q}$ in the following way: we define $\bar{A}_1 = A \setminus A^{*}$, and we define $\bar{A}_{j+1}$ as $\bar{A}_{j}$ after removing the pair from $A \setminus A^{*}$ that was used in replacement in the $j$-th step of our procedure.

As before, for each $B \subseteq A \setminus A^{*}$ and $B' \subseteq A^{*} \setminus A$, and for each attribute $X_i \in X$ we define the gain $g_{i}(B, B')$:
\begin{align*}
g_{i}(B, B') = \sum_{j \in \{1, 2\}}\left(|r_i^j(A) - \pi_i^j| - |r_i^j((A \setminus B) \cup B') - \pi_i^j|\right) \textrm{.}
\end{align*}
Similarly as in the proof of \autoref{thm:localSearch1}, we observe that for $X_i \notin X_{\mathrm{aex}}$ the negative gains cumulate: i.e., that for each sequences of disjoint sets $B_1$, $B_2, \dots$, $B_s$ and $B_1'$, $B_2', \dots$, $B_s'$ such that for every $1 \leq j \leq s$, $B_j \subseteq A \setminus A^{*}$, $B_j' \subseteq A^{*} \setminus A$, and $|B_j| = |B_j'| \leq 2$ we have that:
\begin{align}\label{eq:negativeCummulate2}
g_{i}(\bigcup_jB_j, \bigcup_jB_j') \leq \sum_j g_{i}(B_j, B_j') \textrm{.}
\end{align}
Why is this the case? If $X_i \notin X_{\mathrm{aex}}$, then the distance between $A$ and the target distribution on attribute $X_i$ is at least equal to $2\cdot\frac{2}{k}$. In other words: $|r_i^0(A) - \pi_i^0| \geq \frac{2}{k}$ and $|r_i^1(A) - \pi_i^1| \geq \frac{2}{k}$. Without loss of generality let us assume that $r_i^0(A) - \pi_i^0 \geq \frac{2}{k}$. Since each set $B_j$ and each set $B_j'$ has at most two elements, replacing $B_j$ with $B_j'$ can change the distance between $A$ and the target distribution, for each attribute, by at most $\frac{2}{k}$. Consequently, if $g_{i}(B_j, B_j')$ is negative, then it means that replacing $B_j$ with $B_j'$ makes the difference $r_i^0(A) - \pi_i^0$ even greater. Thus, each such replacement with the negative gain $g$ causes $A$ to move further from the target distribution by the value $g$. Naturally, each replacement with the positive gain $g$ causes $A$ to move closer to the target distribution by at most $g$. Consequently, after the sequence of replacement $\cup_j B_j \leftrightarrow B_j'$ the distance on the attribute $X_i$ cannot improve by more than $\sum_j g_{i}(B_j, B_j')$.

In contrast to the proof of \autoref{thm:localSearch1}, we note that here we require that $X_i \notin X_{\mathrm{aex}}$ instead of $X_i \notin X_{\mathrm{ex}}$---the above observation is not valid if $X_i \in X_{\mathrm{aex}}$ even if $X_i \notin X_{\mathrm{ex}}$.\footnote{
Consider an example in which $\pi_i^1 = \frac{1}{k}$ and $r_i^1(A) = \frac{2}{k}$. Let us consider sets $B=\{b_1, b_2\}, B'=\{b_1', b_2'\}, C=\{c_1, c_2\}, C'=\{c_1', c_2'\}$ such that: $X_i(c_1) = X_i(c_2) = X_i(b_1') = X_i(b_2') = d_i^1$, and $X_i(c_1') = X_i(c_2') = X_i(b_1) = X_i(b_2) = d_i^2$, Thus, we have that:
\begin{itemize}
\item Replacing $B$ with $B'$ results with $r_i^1(A) = \frac{4}{k}$.
\item Replacing $C$ with $C'$ results with $r_i^1(A) = 0$.
\item Replacing $B \cup C$ with $B' \cup C'$ results with $r_i^1(A) = \frac{2}{k}$. 
\end{itemize}
We can repeat this reasoning for $r_i^2(A)$, thus having, $g_{i}(B, B') = -\frac{4}{k}$, $g_{i}(C, C') = 0$ and $g_{i}(B \cup C, B' \cup C') = 0$.
}

\begin{table}[t!]
\centering
\begin{tabular}{ l | c | c | c | c | c | c | c }
  & $X_{i} = X_1$ & $X_{i} = X_2$ & $X_{i} = X_3$ & $X_{i} = X_4$ & $X_{i} = X_5$ & $X_{i} = X_6$ & $X_{i} = X_7$ \\
  \hline
  $X_{i}(a_1 \leftrightarrow a_1')$ & 1 & 1 & 1 & 1 & 0 & 0 & -1 \\
  \hline
  $X_{i}(a_2 \leftrightarrow a_2')$ & -1 & -1 & 1 & 0 & 0 & 1 & 0 \\
  \hline
  $X_{i}(a_3 \leftrightarrow a_3')$ & 0 & -1 & -1 & 0 & 1 & 0 & 1 \\
  \hline
  $X_{i}(a_4 \leftrightarrow a_4')$ & -1 & 1 & -1 & -1 & 1 & 0 & -1
\end{tabular}
\caption{An example illustrating the concept of annihilating pairs. In this example we have $X_{\mathrm{ex}} = \{X_1, X_2, X_3, X_4, X_5, X_6, X_7\}$ and $\bar{A}_{1} = \{a_1, a_2, a_3, a_4\}$. The cell in row ``$X_{i}(a_j \leftrightarrow a_j')$" for $j \in [4]$ and in column ``$X_{i} = X_\ell$'' for $\ell \in [7]$ denotes the value of $X_{\ell}(a_j \leftrightarrow a_j')$.
We recall that $X_{i}(a_i \leftrightarrow a_i') = 1$ if replacing $a_i$ with $a_i'$ moves $A$ further from the target distribution in one direction and $X_{i}(a_i \leftrightarrow a_i') = -1$ if replacing $a_i$ with $a_i'$ moves $A$ further from the target distribution in the other direction. Here, we have $W_1(X_1) = \{\big((a_1, X_1), (a_2, X_1)\big), \big((a_1, X_1), (a_4, X_1)\big)\}$, $W_1(X_2) = \{\big((a_1, X_2), (a_2, X_2)\big), \big((a_1, X_2), (a_3, X_2)\big)\}$, $W_1(X_3) = \{\big((a_1, X_3), (a_3, X_3)\big), \big((a_1, X_3), (a_4, X_3)\big), \big((a_2, X_3), (a_3, X_3)\big), \big((a_2, X_3), (a_4, X_3)\big)\}$, etc. Further, $W_1 = W_1(X_1) \cup W_1(X_2) \cup W_1(X_3) \cup W_1(X_4) \cup W_1(X_5) \cup W_1(X_6) \cup W_1(X_7)$. There are many choices for the set $W$, but it must hold that $P = |W| = 6$; we give the following example: $W =$ $\{\big((a_1, X_1), (a_2, X_1)\big)$, $\big((a_1, X_2), (a_2, X_2)\big)$, $\big((a_1, X_3), (a_3, X_3)\big)$, $\big((a_2, X_3), (a_4, X_3)\big)$, $\big((a_1, X_4), (a_4, X_4)\big)$, $\big((a_1, X_7), (a_3, X_7)\big)\}$. } 
\label{table:annihilatingPairsExample}
\end{table}

Next, for each $\bar{A}_{j}$, and each attribute $X_i \in X_{\mathrm{ex}}$, we define a set $W_j$ of annihilating pairs as:
\begin{align*}
W_j(X_i) = \left\{((a_x, X_i), (a_y, X_i)): a_x \in \bar{A}_{j}; a_y \in \bar{A}_{j}; x < y; X_{i}(a_x \leftrightarrow a_x') = -X_{i}(a_y \leftrightarrow a_y') \right\} \textrm{.}
\end{align*}
Intuitively, if $((a_x, X_i), (a_y, X_i)) \in W_j$, then both replacing $a_x$ with $a_x'$ and replacing $a_y$ with $a_y'$ move the original set $A$ (i.e., the set before any of the replacements) further from the target distribution for the attribute $X_i$, but replacing $\{a_x, a_y\}$ with $\{a_x', a_y'\}$ does not change the distance of $A$ from the target distribution for the attribute $X_i$.

For each $j$, we set $W_j = \cup_{i \in X_{\mathrm{ex}}}W_j(X_i)$.
Let us denote by $P$ the number of annihilated pairs of candidates considered in the process of replacing candidates from $A \setminus A^{*}$ with candidates from $A^{*} \setminus A$.
Formally, $P$ is the size of the maximal subset $W \subseteq W_1$ composed of disjoint annihilating pairs,
i.e., for each $i \leq p$, for each $a_x$, and for each $a_y$, if $((a_x, X_i), (a_y, X_i)) \in W$ then there exists no $b \neq a_y$ such that $((a_x, X_i), (b, X_i)) \in W$ or $((b, X_i), (a_x, X_i)) \in W$.
From \autoref{lemma:localSearch2HelpingLemma}, after defining each bucket $X_i$ as containing $x_i$ white balls and $y_i$ black balls, where $x_i$ (respectively, $y_i$) is the number of candidates $a_j \in \bar{A_1}$ with the value $X_{i}(a_j \leftrightarrow a_j')$ equal to 1 (respectively, -1), it follows that $W_1 \geq \frac{P^2}{|X_{\mathrm{ex}}|}$. The concept of annihilating pairs is explained on example in \autoref{table:annihilatingPairsExample}.

We are now ready to describe the way in which we select pairs from $A \setminus A^{*}$ in our procedure. In each step $j$, the pair $(a_{j, 1}, a_{j, 2})$ from $A \setminus A^{*}$ is selected in the following way. For each candidate $a$ let $s_{j, 1}(a)$ be the number of pairs $p$ in $W_j$ such that $p = ((a, \cdot), (\cdot, \cdot))$  or $p = ((\cdot, \cdot), (a, \cdot))$, let $a_{j, 1}$ be such that $s_{j,1}(a_j) = \max_{a \in \bar{A_j}} s_{j,1}(a)$, and let $s_{j,1} = s_{j,1}(a_j)$.
Next, for each candidate $b$ let $s_{j,2}(b)$ be the number of pairs $p$ in $W_j$ such that $p = ((a_{j, 1}, \cdot), (b, \cdot))$ or $p = ((b, \cdot), (a_{j, 1}, \cdot))$, let $a_{j, 2}$ be such that $s_{j,2}(b) = \max_{b \in \bar{A_j}} s_{j,2}(b)$, and let $s_{j,2} = s_{j,2}(a_{j, 2})$.

Let us consider the procedure described above on the example from \autoref{table:annihilatingPairsExample}. The candidate $a_1$ belongs to 8 pairs in $W_1$ ($a_1$ belongs to 2 pairs for attribute $X_1$, $X_2$, and $X_3$, and to one pair for attributes $X_4$ and $X_7$), thus: $s_{1,1}(a_1) = 8$. Moreover, $s_{1,1}(a_2) = 5$, $s_{1,1}(a_3) = 6$, and $s_{1,1}(a_4) = 7$. Consequently, $a_1$ will be the candidate that will replaced with $a_1'$ in the first step: $a_{j, 1} = a_1$ and $s_{j,1} = 8$. Further, $s_{1,2}(a_2) = 2$ (there are two annihilating pairs including $a_1$ and $a_2$, i.e.,: $\big((a_1, X_1), (a_2, X_1)\big)$ and $\big((a_1, X_2), (a_2, X_2)\big)$); similarly: $s_{1,2}(a_3) = 3$, and $s_{1,2}(a_4) = 3$. Thus, an arbitrary of the two candidates, $a_3$ and $a_4$, say $a_3$, will be the second candidate that will be replaced with $a_3'$ in the first step. In the second step only two candidates, $a_2$ and $a_4$, are left, so both will be replaced with $a_2'$ and $a_4'$ in the second step. Nevertheless, let us illustrate our definitions also in the second step of the replacement procedure. The set $\bar{A_2}$ consists of two remaining candidates: $a_2$ and $a_4$. We have $W_2 = \{\big((a_2, X_2), (a_4, X_2)\big), \big((a_2, X_3), (a_4, X_3)\big)\}$. Naturally, $s_{2,1}(a_2) = s_{2,1}(a_4) = s_{2,2}(a_2) = s_{2,2}(a_4) = 2$.
 
We want now to derive bounds on the values $s_{j, 1}$ and $s_{j, 2}$. The following inequalities hold:
\begin{enumerate}
\item $s_{j,1} \geq \frac{2|W_j|}{2q_f-2(j-1)}$ for each $1 \leq j \leq q$.

$W_j$ contains pairs of candidates belonging to $\bar{A}_{j}$. $\bar{A}_{1}$ has $2q_f$ candidates, and $\bar{A}_{j+1}$ is obtained from $\bar{A}_{j}$ by removing two candidates. Consequently, $\bar{A}_{j}$ has $2q_f-2(j-1)$ candidates, and thus, $W_j$ contains pairs of $2q_f-2(j-1)$ different candidates. From the pigeonhole principle it follows that there exists a candidate that belongs to at least $\frac{2|W_j|}{2q_f-2(j-1)}$ pairs. Naturally, we also get the weaker constraint: $s_{j,1} \geq \frac{|W_j|}{2q_f-2(j-1)}$.
\item $|W_j| \geq |W_{j-1}| - 2s_{j-1, 1}$ for each $2 \leq j \leq q$.

Each candidate in $W_{j-1}$ belongs to at most $s_{j-1, 1}$ pairs (this follows from the definition of $s_{j-1, 1}$). $W_j$ contains all pairs that $W_{j-1}$ contained, except for the pairs involving $a_{j-1, 1}$, $a_{j-2, 2}$ (to obtain $\bar{A}_{j}$, we removed these two candidates from $\bar{A}_{j-1}$). Consequently, $W_j$ is obtained from $W_{j-1}$ by removing at most $2s_{j-1, 1}$ pairs of candidates.

\item $s_{j,2} \geq \frac{s_{j,1}}{2q_f-2(j-1)-1}$ for each $1 \leq j \leq q$.

In $W_j$, there are $s_{j,1}$ pairs of candidates involving $a_{j, 1}$. As we noted before, $W_j$ contains pairs of $2q_f-2(j-1)$ different candidates. Thus, in $W_j$, $a_{j, 1}$ is paired with at most $2q_f-2(j-1)-1$ candidates. From the pigeonhole principle it follows that $a_{j, 1}$ must be paired with some candidate at least $\frac{s_{j,1}}{2q_f-2(j-1)-1}$ times.
\end{enumerate}
From \autoref{lemma:constraintsEstimation} we get that:
\begin{align}\label{eq:elementsNumEstimation}
\sum_{j=1}^q s_{j,2} \geq \frac{|W_1|\ln q}{4q} \textrm{.}
\end{align}

\begin{figure}[tb]
\hspace{-1cm}\includegraphics[scale=1.0]{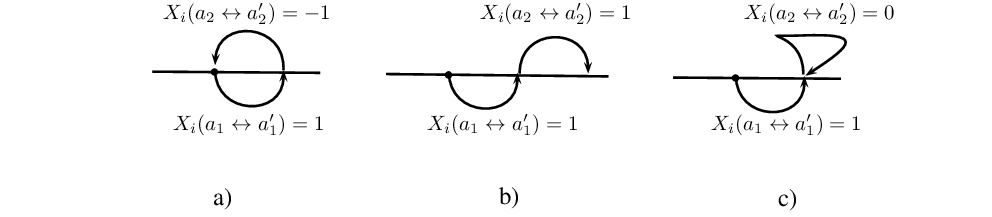}
\caption{Figure illustrating that for $X_i \in X_{\mathrm{ex}}$, $g_{i}(\{a_1, a_2\}, \{a_1', a_2'\})$ is greater than $(g_{i}(a_1, a_1') + g_{i}(a_2, a_2'))$ if and only if $((a_1, X_i), (a_2, X_i))$ is an annihilating pair. The figure presents 3 scenarios: a) $((a_1, X_i), (a_2, X_i))$ is an annihilating pair. Both replacing $a_1$ with $a_1'$ and replacing $a_2$ with $a_2'$ moves us further from the target distribution for attribute $X_i$ (the target distribution is marked as a black dot), thus $g_{i}(a_1, a_1') = -\frac{2}{k}$ and $g_{i}(a_2, a_2') = -\frac{2}{k}$. However these changes annihilate, and $g_{i}(\{a_1, a_2\}, \{a_1', a_2'\}) = 0$. b) $g_{i}(a_1, a_1') = -\frac{2}{k}$ and $g_{i}(a_2, a_2') = -\frac{2}{k}$, but these changes do not annihilate, and thus: $g_{i}(\{a_1, a_2\}, \{a_1', a_2'\}) = -\frac{4}{k}$. c) $g_{i}(a_1, a_1') = -\frac{2}{k}$ and $g_{i}(a_2, a_2') = 0$, if at least one change does not move the solution against the target distribution, the changes do not annihilate, and $g_{i}(\{a_1, a_2\}, \{a_1', a_2'\}) = g_{i}(a_1, a_1') + g_{i}(a_2, a_2')$.}
\label{fig:proof1}
\end{figure}

Before we proceed further let us make three observations regarding annihilating pairs. First, we note that for each $X_i \in X_{\mathrm{ex}}$, and each $a_x$ and $a_y$, if the value $g_{i}(\{a_x, a_y\}, \{a_x', a_y'\})$ is different from $(g_{i}(a_x, a_x') + g_{i}(a_y, a_y'))$ than it is greater from $(g_{i}(a_x, a_x') + g_{i}(a_y, a_y'))$ by $\frac{4}{k}$. We also note that $g_{i}(\{a_x, a_y\}, \{a_x', a_y'\})$ is greater than $(g_{i}(a_x, a_x') + g_{i}(a_y, a_y'))$ if and only if the changes $X_{i}(a_x \leftrightarrow a_x')$ and $X_{i}(a_y \leftrightarrow a_y')$ annihilate (this is illustrated in \autoref{fig:proof1}). Further, we recall that the value $s_{j,2}$ counts all attributes for which $a_{j,1}$ and $a_{j,2}$ constitute an annihilating pair. 
Thus, for each $1 \leq j \leq q$::
\begin{align}\label{eq:firstObservation}
\sum_{i \in X_{\mathrm{ex}}} g_{i}(\{a_{j,1}, a_{j,2}\}, \{a_{j,1}', a_{j,2}'\}) = \sum_{i \in X_{\mathrm{ex}}} \left(g_{i}(a_{j,1}, a_{j,1}') + g_{i}(a_{j,2}, a_{j,2}') \right) + s_{j,2}\frac{4}{k}
\end{align}

\begin{figure}[tb]
\hspace{-0.5cm}\includegraphics[scale=1.0]{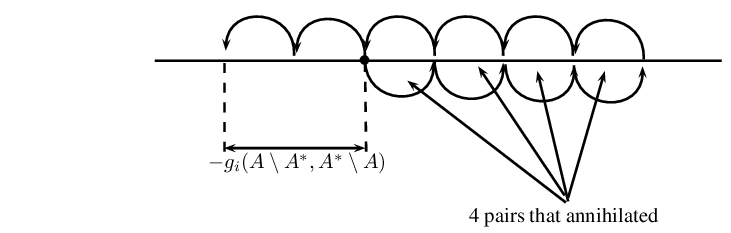}
\caption{Figure illustrating the effect of replacing 10 candidates for an attribute $X_i \in X_{\mathrm{ex}}$. Each replacement imposes a negative gain: $g_{i}(a_j, a_j') = -\frac{2}{k}$ for $1\leq j \leq 10$. Thus, $\sum_{a \in A \setminus A^{*}} g_i(a, a') = -\frac{20}{k}$. In this example four pairs annihilated, and, consequently, $g_i(A \setminus A^{*}, A^{*} \setminus A) = -\frac{4}{k}$.}
\label{fig:proof2}
\end{figure}

Our second observation is similar in spirit to the first one. We note that for each $X_i \in X_{\mathrm{ex}}$:
\begin{align*}
g_i(A \setminus A^{*}, A^{*} \setminus A) - \sum_{a \in A \setminus A^{*}} g_i(a, a') = \text{the number of pairs that annihilated for $X_i$ $\times$ $\frac{4}{k}$.}
\end{align*}
The above equality is illustrated in \autoref{fig:proof2}. As a consequence, we get that:
\begin{align*}
\sum_{X_i \in X_{\mathrm{ex}}}\Big(g_i(A \setminus A^{*}, A^{*} \setminus A) - \sum_{a \in A \setminus A^{*}} g_i(a, a')\Big) =  \text{the number of pairs that annihilated $\times$ $\frac{4}{k}$.}
\end{align*}
We recall that after the replacement procedure $A$ can differ from $A^{*}$ with at most one element, hence, having distance to the optimal distribution at most equal to $|X|\frac{2}{k}$. Thus:
\begin{align}\label{eq:allAnnihilatedPairs}
\sum_{X_i \in X_{\mathrm{ex}}}\Big(g_i(A \setminus A^{*}, A^{*} \setminus A) - \sum_{j=1}^q \left( g_{i}(a_{j,1}, a_{j,1}') + g_{i}(a_{j,2}, a_{j,2}') \right)\Big) \leq P \cdot \frac{4}{k} + |X|\frac{2}{k} \textrm{.}
\end{align}

Our third observation says that:
\begin{align}\label{eq:aexEstimation}
\sum_{X_i \in X_{\mathrm{aex}} \setminus X_{\mathrm{ex}}} g_i(A \setminus A^{*}, A^{*} \setminus A) - \sum_{X_i \in X_{\mathrm{aex}} \setminus X_{\mathrm{ex}}} \sum_{j=1}^q  g_{i}(\{a_{j,1}, a_{j,2}\}, \{a_{j,1}', a_{j,2}'\}) \leq \left|X_{\mathrm{aex}} \setminus X_{\mathrm{ex}}\right| \textrm{.}
\end{align}
Where does \autoref{eq:aexEstimation} come from? Let us use the geometric interpretation, like the one from \autoref{fig:proof2}. Let us consider an $X_i$, $X_i \in X_{\mathrm{aex}}$. For $X_i$, $A$ lies in a distance of $\frac{2}{k}$ on the left or on the right from the target distribution. Without loss of generality, let us assume it lies on the right. Now, if $g_{i}(\{a_{j,1}, a_{j,2}\}, \{a_{j,1}', a_{j,2}'\}) < 0$ then replacing $(a_{j,1}, a_{j,2})$ with $(a_{j,1}', a_{j,2}')$ moves the current solution right. If $g_{i}(\{a_{j,1}, a_{j,2}\}, \{a_{j,1}', a_{j,2}'\}) = \frac{2}{k}$, then replacing $(a_{j,1}, a_{j,2})$ with $(a_{j,1}', a_{j,2}')$ moves the current solution by $\frac{2}{k}$ on left. If $g_{i}(\{a_{j,1}, a_{j,2}\}, \{a_{j,1}', a_{j,2}'\}) = 0$, then replacing $(a_{j,1}, a_{j,2})$ with $(a_{j,1}', a_{j,2}')$ either does not move the solution or moves it by $\frac{4}{k}$ on left.

Let us define $y_i =  g_i(A \setminus A^{*}, A^{*} \setminus A) - \sum_{j=1}^q  g_{i}(\{a_{j,1}, a_{j,2}\}, \{a_{j,1}', a_{j,2}'\})$. If the solution moves $q$ times to the right, then the total gain $-\sum_{j=1}^q  g_{i}(\{a_{j,1}, a_{j,2}\}, \{a_{j,1}', a_{j,2}'\})$ will be maximised, achieving $q\frac{4}{k}$. In such a case however, the value $g_i(A \setminus A^{*}, A^{*} \setminus A)$ will be equal to $-q\frac{4}{k}$, and thus the value $y_i$ will be equal to 0. After some consideration, the reader will see that the value $y_i$ is maximised if the current solution moves $\frac{q}{2}$ times right and $\frac{q}{2}$ times left, each time by the value of $\frac{4}{k}$. This way, the moves to the right induce the total gain of $\frac{q}{2}\cdot\frac{4}{k}$, the moves to the left induce the zero gain, but as a consequence, the current solution for $X_i$ does not change ($g_i(A \setminus A^{*}, A^{*} \setminus A) = 0$). Thus, for each $X_i \in X_{\mathrm{aex}}$, $y_i$ is upper bounded by $\frac{q}{2}\cdot\frac{4}{k} \leq 1$, which proves \autoref{eq:aexEstimation}.

We can further proceed with the proof by observing that from the condition in the ``while'' loop we get that for each $1 \leq j \leq q$:
\begin{align*}
0 &\geq \sum_ig_{i}(\{a_{j,1}, a_{j,2}\}, \{a_{j,1}', a_{j,2}'\}) \\
  &\geq \sum_{i \in X_{\mathrm{ex}}} g_{i}(\{a_{j,1}, a_{j,2}\}, \{a_{j,1}', a_{j,2}'\}) + \sum_{i \notin X_{\mathrm{ex}}} g_{i}(\{a_{j,1}, a_{j,2}\}, \{a_{j,1}', a_{j,2}'\}) \\
  & \textrm{From~\eqref{eq:firstObservation}:} \\
  &\geq \sum_{i \in X_{\mathrm{ex}}} \left( g_{i}(a_{j,1}, a_{j,1}') + g_{i}(a_{j,2}, a_{j,2}') \right) + s_{j,2}\frac{4}{k} + \sum_{i \notin X_{\mathrm{ex}}} g_{i}(\{a_{j,1}, a_{j,2}\}, \{a_{j,1}', a_{j,2}'\}) \textrm{.}
\end{align*}
Thus, we get:
\begin{align}\label{eq:noGainInEx2}
-\sum_{i \in X_{\mathrm{ex}}} \left( g_{i}(a_{j,1}, a_{j,1}') + g_{i}(a_{j,2}, a_{j,2}') \right) - \frac{4}{k}s_{j,2} >  + \sum_{i \notin X_{\mathrm{ex}}} g_{i}(\{a_{j,1}, a_{j,2}\}, \{a_{j,1}', a_{j,2}'\}) \textrm{.}
\end{align}

Next, we give the following sequence of inequalities:
 \begin{align*}
g &= \sum_i g_i(A \setminus A^{*}, A^{*} \setminus A) \\
  &= \sum_{X_i \in X_{\mathrm{ex}}} g_i(A \setminus A^{*}, A^{*} \setminus A) + \sum_{X_i \in X_{\mathrm{aex}} \setminus X_{\mathrm{ex}}} g_i(A \setminus A^{*}, A^{*} \setminus A) + \sum_{X_i \notin X_{\mathrm{aex}}} g_i(A \setminus A^{*}, A^{*} \setminus A) 
  \end{align*}

From~\autoref{eq:negativeCummulate2}, for all $i \notin X_{aex}$, we have $g_i(A \setminus A^{*}, A^{*} \setminus A) \leq \sum_{a \in A \setminus A^{*}} g_i(a, a')$. Since the set $A \setminus A^{*}$ and $\bigcup_{j=1}^q \{a_{j,1}, a_{j,2}\}$ can differ by at most one candidate (which induces distance $\frac{2|X|}{k}$ to the optimal solution), we have that
\begin{align*}
\sum_{X_i \notin X_{\mathrm{aex}}} g_i(A \setminus A^{*}, A^{*} \setminus A) \leq \sum_{X_i \notin X_{\mathrm{aex}}} \sum_{j=1}^q  g_{i}(\{a_{j,1}, a_{j,2}\}, \{a_{j,1}', a_{j,2}'\}) + \frac{2|X|}{k} \textrm{.}
\end{align*}
And, as a consequence:
\begin{align*}
g  \leq\;\;& \sum_{X_i \in X_{\mathrm{ex}}} g_i(A \setminus A^{*}, A^{*} \setminus A) + \sum_{X_i \in X_{\mathrm{aex}} \setminus X_{\mathrm{ex}}} g_i(A \setminus A^{*}, A^{*} \setminus A) \\
           &  + \sum_{X_i \notin X_{\mathrm{aex}}} \sum_{j=1}^q  g_{i}(\{a_{j,1}, a_{j,2}\}, \{a_{j,1}', a_{j,2}'\}) + \frac{2|X|}{k} \\
   \leq\;\;& \sum_{X_i \in X_{\mathrm{ex}}} g_i(A \setminus A^{*}, A^{*} \setminus A) + \sum_{X_i \in X_{\mathrm{aex}} \setminus X_{\mathrm{ex}}} g_i(A \setminus A^{*}, A^{*} \setminus A) \\
           & + \sum_{X_i \notin X_{\mathrm{ex}}} \sum_{j=1}^q  g_{i}(\{a_{j,1}, a_{j,2}\}, \{a_{j,1}', a_{j,2}'\}) - \sum_{X_i \in X_{\mathrm{aex}} \setminus X_{\mathrm{ex}}} \sum_{j=1}^q  g_{i}(\{a_{j,1}, a_{j,2}\}, \{a_{j,1}', a_{j,2}'\}) + \frac{2|X|}{k} \textrm{.}
\end{align*}
From \autoref{eq:aexEstimation} we get:
\begin{align*}
g \leq \sum_{X_i \in X_{\mathrm{ex}}} g_i(A \setminus A^{*}, A^{*} \setminus A) + \sum_{X_i \notin X_{\mathrm{ex}}} \sum_{j=1}^q  g_{i}(\{a_{j,1}, a_{j,2}\}, \{a_{j,1}', a_{j,2}'\}) + \frac{2|X|}{k} + \left|X_{\mathrm{aex}} \setminus X_{\mathrm{ex}}\right| \textrm{.}
\end{align*}
From \autoref{eq:noGainInEx2}:
\begin{align*}
&g \leq \frac{2|X|}{k} + \left|X_{\mathrm{aex}} \setminus X_{\mathrm{ex}}\right|  + \sum_{X_i \in X_{\mathrm{ex}}} g_i(A \setminus A^{*}, A^{*} \setminus A) \\
& \hspace{2cm} - \sum_{X_i \in X_{\mathrm{ex}}} \sum_{j=1}^q \left( g_{i}(a_{j,1}, a_{j,1}') + g_{i}(a_{j,2}, a_{j,2}') \right) - \frac{4}{k} \sum_j s_{j,2} \textrm{.}
\end{align*}
From \autoref{eq:elementsNumEstimation}:
\begin{align*}
&g \leq \frac{2|X|}{k} + \left|X_{\mathrm{aex}} \setminus X_{\mathrm{ex}}\right| - \frac{|W_1|\ln q}{4 q}\cdot\frac{4}{k} \\
& \hspace{2cm} + \sum_{i \in X_{\mathrm{ex}}} \left( g_i(A \setminus A^{*}, A^{*} \setminus A) - \sum_{j=1}^q \left( g_{i}(a_{j,1}, a_{j,1}') + g_{i}(a_{j,2}, a_{j,2}') \right)\right) 
\end{align*}
From \autoref{eq:allAnnihilatedPairs}:
\begin{align*} 
g \leq \frac{4|X|}{k} + \left|X_{\mathrm{aex}} \setminus X_{\mathrm{ex}}\right| - \frac{|W_1|\ln q}{kq} + P\frac{4}{k} \textrm{.}
\end{align*}
As we noted before, from \autoref{lemma:localSearch2HelpingLemma}, we have that $W_1 \geq \frac{P^2}{|X_{\mathrm{ex}}|}$. Thus:
\begin{align*} 
g \leq \frac{4|X|}{k} + \left|X_{\mathrm{aex}} \setminus X_{\mathrm{ex}}\right| + \frac{4}{k}\left(P - \frac{P^2\ln q}{4|X_{\mathrm{ex}}|q}\right) \textrm{.}
\end{align*}
Since $q \leq \frac{k}{2}$, and since the function $\frac{\ln x}{x}$ is decreasing for $x \geq 1$:   
\begin{align*} 
g \leq \frac{4|X|}{k} + \left|X_{\mathrm{aex}} \setminus X_{\mathrm{ex}}\right| + \frac{4}{k}\left(P - \frac{P^2\ln (k/2)}{2|X_{\mathrm{ex}}|k}\right) 
\end{align*}
The function $f(P) = P - \frac{P^2\ln (k/2)}{2|X_{\mathrm{ex}}|k}$ takes its maximum for $P = \frac{|X_{\mathrm{ex}}|k}{\ln (k/2)}$. Thus:
\begin{align*} 
g \leq \frac{4|X|}{k} +\left|X_{\mathrm{aex}} \setminus X_{\mathrm{ex}}\right| + \frac{4}{k} \cdot \frac{|X_{\mathrm{ex}}|k}{2\ln (k/2)}
= \frac{4|X|}{k} + \left|X_{\mathrm{aex}} \setminus X_{\mathrm{ex}}\right| + \frac{2|X_{\mathrm{ex}}|}{\ln (k/2)} \textrm{.}
\end{align*}
Since our local-search algorithm for $\ell = 2$ also tries to perform local swaps on single candidates, we can repeat the analysis from the proof of \autoref{thm:localSearch1}. Thus, using \autoref{eq:sequenceFromTheorem1} from there, we get that $g \leq 2|X_{\mathrm{ex}}|$, and as a consequence: $\left(\frac{1}{2} - \frac{1}{\ln (k/2)}\right)g \leq |X_{\mathrm{ex}}| - \frac{2|X_{\mathrm{ex}}|}{\ln (k/2)}$.

For each attribute $X_i \in X \setminus X_{\mathrm{aex}}$ the distance from $A$ and the target distribution is bounded by 2. For $X_i \in X_{\mathrm{aex}}$ this distance is bounded by $\frac{2}{k}$. Thus, we get that $g \leq 2(|X| - |X_{\mathrm{ex}}| - |X_{\mathrm{aex}} \setminus X_{\mathrm{ex}}|) + |X|\frac{2}{k}$, and so:
\begin{align*}
g + \left(\frac{1}{2} - \frac{1}{\ln (k/2)}\right)g + \frac{1}{2}g \leq& \frac{4|X|}{k} + \left|X_{\mathrm{aex}} \setminus X_{\mathrm{ex}}\right| + \frac{2|X_{\mathrm{ex}}|}{\ln (k/2)} \\
&+ |X_{\mathrm{ex}}| - \frac{2X_{\mathrm{ex}}|}{\ln (k/2)} \\
&+  (|X| - |X_{\mathrm{ex}}| - |X_{\mathrm{aex}} \setminus X_{\mathrm{ex}}|) + |X|\frac{2}{k} \\
=&\;\; |X| + \frac{6|X|}{k}
\end{align*}
Finally, we get:
\begin{align*}
g \leq \frac{\ln (k/2)}{2\ln (k/2) - 1}\left(|X| + \frac{6|X|}{k} \right) \textrm{}
\end{align*}
which completes the proof.
\end{proof}

Since  a brute-force algorithm can be used to compute an optimal solution for small values of $k$, \autoref{thm:localSearch2} implies that for every $\epsilon > 0$ we can achieve an additive approximation of $\frac{1}{2}(|X| + \epsilon)$. That is, we can guarantee that the solution returned by our algorithm will be at least 4 times better than a solution that is arbitrarily bad on each attribute. A natural open question 
is whether the local search algorithm 
achieves even better approximation guarantees for larger values of $\ell$.

One may argue that the restriction to natural target distributions is quite strong. However, for a given vector of target distributions $\pi$, we can easily find a vector $\rho$ of target natural distributions such that \mbox{$\sum_{i, j}\big|\rho_i^j -  \pi_i^j\big|\leq \frac{2|X|}{k}$}. For instance for $k=5$ and $p = |X| = 3$ the distribution 
\begin{align*}
\pi = \left(\left(\frac{2}{5} + \frac{1}{10}, \frac{3}{5} - \frac{1}{10}\right), \left(\frac{1}{5} + \frac{1}{7}, \frac{4}{5} - \frac{1}{7}\right), \left(\frac{1}{6}, 1 - \frac{1}{6}\right)\right)
\end{align*}
is not natural, yet there exists a natural distribution
\begin{align*}
\rho = \left(\left(\frac{2}{5}, \frac{3}{5}\right), \left(\frac{1}{5}, \frac{4}{5}\right), \left(0, 1\right)\right)
\end{align*}
such that \mbox{$\sum_{i, j}\big|\rho_i^j -  \pi_i^j\big|\leq \frac{2|X|}{k}$}. Thus, the results from \autoref{thm:localSearch1} and \autoref{thm:localSearch2} can be modified by providing approximation ratios that are worse by an additive value of $\frac{2|X|}{k}$ but valid for arbitrary target distributions. Again, since an optimal solution can easily be computed for small values of $k$, we can get approximation guarantees arbitrarily close to the ones given by \autoref{thm:localSearch1} and \autoref{thm:localSearch2}, even for non-natural target distributions.

Below we show a lower bound of $\frac{2|X|}{7}$ for the approximation ratio of the local search algorithm from \autoref{alg:local-search} with $\ell = 2$.

\begin{example}
Consider $7$ binary attributes $X_1, \dots, X_{7}$, and the set of $12p$ candidates $C = \{a_1, \dots, a_{2p}, a_1', \dots, a_{2p}', b_1, \dots, b_{2p}, b_1', \dots, b_{2p}', c_1, \dots, c_{2p}, c_1', \dots, c_{2p}'\}$. For each $i \in [k]$, we have:

\begin{center}
\begin{tabular}{|l|c|c|c|c|c|c|c|}
  \hline
  & $X_1$ & $X_2$ & $X_3$ & $X_4$ & $X_5$ & $X_6$ & $X_7$ \\ \hline \hline
  $a_i$ & 1 & 0 & 1 & 1 & 0 & 0 & 1 \\ \hline
  $a_i'$ & 0 & 1 & 0 & 0 & 1 & 1 & 1 \\ \hline \hline
  $b_i$ & 0 & 0 & 0 & 0 & 0 & 0 & 0 \\ \hline
  $b_i'$ & 0 & 0 & 1 & 1 & 1 & 1 & 0 \\ \hline \hline
  $c_i$ & 1 & 1 & 1 & 1 & 0 & 0 & 0 \\ \hline
  $c_i'$ & 1 & 1 & 0 & 0 & 1 & 1 & 0 \\ \hline
\end{tabular}
\end{center}
We note that for each candidate the value of the attribute $X_3$ is the same as of $X_4$ and the value of the attribute $X_5$ is the same as of $X_6$.
For $i \in \{1, 2, 3, 4, 5, 6\}$ let $\pi_{i}^0 = \pi_{i}^1 = \frac{1}{2}$, and let $\pi_{7}^0 = 1 - \pi_{7}^1 = 1$.

Let us fix $k = 4p$. It can be easily checked that the set consisting of $p$ copies of candidates $b_i$, $b_i'$, $c_i$, $c_i'$ is a perfect committee. On the other hand, the set $A$ consisting of $2p$ copies of candidates $a_i$ and $a_i'$ is locally optimal. Indeed, replacing candidate $a_i$ or $a_i'$ with $b_i$ or $b_i'$ moves the solution closer to the target distribution on $X_7$, but  the further from the target distribution on $X_1$ or $X_2$. The same situation happens if we replace candidates $a_i$ or $a_i'$ with $c_i$ or $c_i'$. If we replace two $a$-candidates with the pair consisting of one $b$-candidate ($b_i$ or $b_i'$) and one $c$-candidate ($c_i$ or $c_i'$), then such a replacement will move the solution closer by $\nicefrac{4}{k}$ to the target distribution on $X_7$, but will move the solution further by $\nicefrac{2}{k}$ on two attributes from $\{X_3, X_4, X_5, X_6\}$.

Finally, $\sum_{i, j} \big|r_i^j(A) -  \pi_i^j\big| = 2p = \frac{2}{7}|X|$. \qed
\end{example}

\section{Related Work}\label{related}

Our model is related to the following research areas:\smallskip

\subsection{Apportionment for Party-List Representation Systems} 
As we already pointed out, classical apportionment methods correspond to the restriction of our model to a single attribute (albeit with a different motivation). See the work of Balinski and Young~\cite{Balinski01} for a survey. 
While voting on multi-attribute domains and multiwinner elections have led to significant research effort in computational social choice, this is less the case for party-list representation systems. Ding and Lin \cite{DingL14} studied a game-theoretic model for a party-list  proportional representation system under specific assumptions, and show that computing the Nash equilibria of the game is {\sf NP}-hard. 

\subsection{Biapportionment}

The biapportionment setting~\cite{balinski:halshs-00585327,CIS-88648} has some similarities with our multi-attribute proportional representation setting (MAPR). In biapportionment we are given two attributes, one corresponding to parties and the other one to voting districts. The input consists of (1) hard constraints expressing lower and upper bounds on the number of candidates to be elected in each district, and similarly, bounds on the number of candidates to elected from each party; (2) for each district $i$ and party $P_j$, a value $p_{ij}$ corresponding to the number of votes for party $P_j$ in district $i$. (2) induces a soft proportionality constraint: the number of elected candidates from party $P_j$ in district $i$ should be as much as possible proportional to $p_{ij}$. 

There are however substantial differences between biapportionment and MAPR. First, we do not have anything that corresponds to the values $p_{ij}$: while in biapportionment the target composition of the committee consists of a target number of seats for each combination of the two attributes, in MAPR, on the other hand, we have a smaller input consisting of a target number for each value of each attribute.\footnote{Yet, a target number of seats for each combination of two or more attributes could be incorporated to our model by providing an attribute that corresponds to the Cartesian product of the given attributes; but, for combinatorial reasons, this ceases to be realistic for more than two or three attributes.}
The second (and most important)  difference between biapportionment and MAPR is that in MAPR we have a limited supply of available candidates characterised each by a tuple of attribute values: in our words, we focus on the case when the full supply assumption is not satisfied, which not only corresponds to the practical cases we have in mind, but is required in practice when the number of attributes is large. On the other hand, 
in biapportionment, it is implicitly assumed that there are enough candidates so that there always exist a solution satisfying given (often restrictive) hard constraints. Note finally that the computation of biapportionment methods has been investigated in a few recent papers~\cite{LariRS14,PukelsheimRSSS12,SerafiniS12}.


\subsection{Constrained Approval Voting}\label{sec:constrained_av}

Constrained approval voting (CAP) \cite{brams1990constrained,potthoff90} is also close to MAPR.  In CAP there are also multiple attributes, candidates are represented by tuples of attribute values, there is a target composition of the committee and we try to find a committee close to this target. However, there are also substantial differences between MAPR and CAP. First, in CAP, like in biapportionment, the target composition of the committee, exogenously defined, consists of a target number of seats {\em for each combination of attributes} (called a cell), that is, for each $\vec{z} \in D_1 \times \ldots \times D_p$, we have a value $s(\vec{z})$; while in MAPR, as we said above, we have a smaller input consisting of a target number for each value of each attribute. Note that the input in CAP is exponentially large in the number of attributes, which makes it infeasible in practice as soon as this number exceeds a few units (probably CAP was designed for very small numbers of attributes). 
 Second, in CAP, the selection criterion of an optimal committee is made in two consecutive steps: first a set of {\em admissible committees} is defined, and the choice between these admissible committees is made by using approval ballots, and the chosen committee is the admissible committee maximising the sum, over all voters, of the number of candidates approved (there are no target fractions as in MAPR). 
A simple translation of CAP into an integer linear programming problem is given in \cite{potthoff90,straszak1993computer}.

\subsection{Voting on Multi-Attribute Domains and Judgment Aggregation} 

Another interesting degenerated case is when $k = 1$, {\em i.e.}, when we must select a single candidate from the database. The ideal case is when there exists a candidate in the database whose value on each attribute $i$ coincides with the attribute value $x_i^j$. In this case, this candidate should certainly be selected; otherwise, the most representative candidate should be selected, for some measure of representativity. \smallskip

This problem relates to voting in multi-attribute (or combinatorial) domains (cf. the recent survey chapter \cite{LangXia15}). There, the aim is to output a single winning combination of attributes 
given the preferences of voters over combinations of attribute values, generally expressed in some compact form. When $k = 1$, our model can be viewed as a voting problem in a {\em constrained} multi-attribute domain (constrained because not all combinations are feasible). Another important difference is that in voting in multi-attribute domains, the focus is generally on the way of dealing with nonseparable preferences; here, the issue is avoided, as throughout our paper preferences are assumed to be separable.\footnote{Extending our model to nonseparable preferences would consist in expressing preferences such as {\em if the gender ratio is 50-50 then the ideal group ratio is 40-30-30, otherwise 50-25-25}, or else {\em we want a gender ratio 50-50 or a seniority ratio 50-50}. We are not sure whether it is worth developing this generalisation.} \smallskip

Our model also relates to judgment aggregation (see \cite{Endriss16} for a recent survey). In judgment aggregation, there is a set of propositions $\{\varphi_1, \ldots, \varphi_p\}$; the set of consistent (and complete) judgment sets is a subset ${\cal J}$ of $\times_{i=1}^p \{\varphi_i, \neg \varphi_i \}$; a judgment aggregation profile $V = (V_1, \ldots, V_n)$  is a collection of judgment sets from ${\cal J}$; an irresolute judgment aggregation rule $F$ maps a judgment aggregation profile to a 
nonempty subset of ${\cal J}$; such a rule is said to be based on the weighted majoritarian judgment set if its output can be computed from the vector $\alpha_V = (\alpha_1, \ldots, \alpha_p)$, where $\alpha_j$ is the proportion of judgment sets in $V$ which contain $\varphi_i$. 

Now, consider a multi-attribute proportional representation (MAPR) setting where all attributes are binary; we can view each attribute $X_i$ as a proposition $\varphi_i$. Next, for each database candidate $c \in C$, the judgment set $J_c$ is defined by $J_c = \{\varphi_i\colon i \in [p],  X_i(c) = 1\} \cup \{\neg \varphi_i \colon i \in [p],  X_i(c) = 0\}$. The set of consistent judgment sets ${\cal J}_C$ is defined as ${\cal J}_C = \{J_c \ | \ c \in C\}$; in other words, $J \in {\cal J}_C$ is consistent if and only if there is a candidate $c$ in $C$ such that $(X_1(c), \ldots, X_p(c))$ corresponds to $J$.

Finally, let $(\alpha_1, \ldots, \alpha_p) = (\pi_1^1, \ldots, \pi_1^p)$.  Let ${\cal R}$ be an irresolute MAPR rule; ${\cal R}$ induces an irresolute judgment aggregation rule $F_{\cal R}$, based on the weighted majoritarian judgment set, defined by $F_{\cal R}(V) = \{J_c \ | \ \{c\} \in {\cal R}(C, \alpha_V, 1)\}$. Conversely, from a judgment aggregation rule $F$ based on the weighted majoritarian judgment set we can define a MAPR rule ${\cal R}_F$ restricted to $k = 1$, by ${\cal R}_F(C, \alpha, 1) = \{ \{c\}, J_c \in F(V_\alpha) \}$. 
It is interesting to see which judgment aggregation rules correspond to the two MAPR rules we have defined when $k = 1$. 

The {\em median} judgment aggregation rule\footnote{This rule has been introduced independently in several different papers under different names, and it is probably not relevant to cite them here. A recent paper on the median rule, together with an axiomatisation, is \cite{NehringPivato16}.} 
is defined as follows: given a weighted majoritarian judgment set $\alpha_V = (\alpha_1, \ldots, \alpha_p)$ and $J \in {\cal J}$, let $(J | \alpha) = \sum_{i, \varphi_i \in J}  \alpha_i + \sum_{i, \neg \varphi_i \in J}  (1-\alpha_i)$. Then $median(\alpha) = {\rm argmax}_{J \in {\cal J}} (J | \alpha)$.
Now, let $C$ be a candidate database over a domain of binary attributes, and $k = 1$. Given $\alpha_V$, we have $\{c\} \in {\cal R}_{H}(C,\alpha_V,1)$ if $\sum_{i=1}^p \sum_{j = 1,2} |\pi_i^j - X_i(c)|$ is minimum; now, 
\begin{align*}
\sum_{i=1}^p \sum_{j = 1,2} |\pi_i^j - X_i(c)| &= \sum_{i\in[p], X_i(c) = 1} (1-\alpha_i) +  \sum_{i\in[p], X_i(c) = 0} \alpha_i \\
&= p - \left(\sum_{i\in [p], X_i(c) = 1} \alpha_i +  \sum_{i\in[p], X_i(c) = 0} (1-\alpha_i)\right) = p -(J_c | \alpha) \text{,} 
\end{align*}
therefore ${\cal R}_{H}(\alpha_V)$ contains $\{c\}$ if $(J_c | \alpha)$ is maximum, that is, if $J_c \in median(\alpha)$.

The calculations for the multi-attribute d'Hondt rule are similar: 
\begin{align*}
\sum_{i\in[p]} \sum_{j = 1,2} \pi_i^j \mathrm{H}(r_i^j(\{c\}) &= \sum_{i\in[p], X_i(c) = 1} \pi_i^1 + \sum_{i\in[p], X_i(c) = 0} \pi_i^0 \\
                                                     &= \sum_{i\in[p], X_i(c) = 1} \alpha_i + \sum_{i\in[p], X_i(c) = 0} 1-\alpha_i = (J_c | \alpha) \text{.}
\end{align*}
Therefore, ${\cal R}_{dHondt}(\alpha_V)$ contains $\{c\}$ if $J_c \in median(\alpha)$.
In summary:

\begin{observation}
$F_{{\cal R}_{H}}$ and $F_{{\cal R}_{dHondt}}$ coincide with the {\em median} judgment aggregation rule.
\end{observation}

\subsection{Multiwinner (or Committee) Elections}
In multiwinner elections the voters vote directly for candidates and do not consider attributes that characterise them.
Thus, in this literature, the term ``proportional representation'' \cite{ccElection,Monroe95} has a different meaning: these methods are `representative' because each voter feels represented by some member of the elected committee. The computational aspects of full proportional  representation and its extensions have raised a lot of attention lately~\cite{ProcacciaRosenscheinZohar08,fullyProportionalRepr,CornazGalandSpanjaard12,sko-fal-sli:j:multiwinner,LuBoutilier13}. Our study of the properties of multi-attribute proportional representation is close in spirit to the work of Elkind~et~al.~\cite{elk-fal-sko-sli:c:multiwinner-rules}, who gives a normative study of multiwinner election rules.
{\em Budgeted social choice}~\cite{budgetSocialChoice, sko-fal-lan:j:collective} is technically close to committee elections, but it has a different motivation: the aim is to make a collective choice about a set of objects to be consumed by the group (perhaps, subject to some constraints) rather than about the set of candidates to represent voters.

There exists an interesting line of research on multiwinner voting~\cite{journals/geb/Casella05, StorableVotesBook, RePEc:ucp:jpolec:doi:10.1086/670380, paradoxMultipleElections, skow:c:multiwinner-models, LangXia15}, where it is assumed that the elected committee runs a sequence of independent ballots on various issues---for instance consider a parliament voting on issues such as monetary politics, changes to the national-health care system, or educational reforms. Each issue can be represented by an attribute; in this setting our MAPR methods can be used to find a representative committee with respect to its collective views on a certain set of issues.  

\section{Discussion of the Model and of its Possible Extensions}\label{discussion}

In this section we discuss several other approaches to the problem of achieving proportional representation with respect to multiple attributes, and we compare them with the model discussed so far.

\subsection{Lower and Upper Quotas for Attributes}\label{discussion-bounds}

In \autoref{model} we assumed that the input contains a vector of target distributions which describe desirable proportions of values for different attributes in an ideal committee---such an ideal committee might not exist, e.g., because there is not enough diversity within the candidate database (in particular, the candidate database might not satisfy the full supply property), or even if an ideal committee exists it might be computationally infeasible to find one. For this reason we formulated two optimisation metrics which, intuitively, allow one to assess how good are certain committees, and to find committees which are good enough, though not necessary ideal. Thus, it is natural to consider another approach: instead of getting a vector of ideal target distributions, we could assume that for each value of each attribute we are given a lower and an upper bound (also referred to as lower and upper quota, respectively) on the number of committee members with such a value of the respective attribute. For instance, instead of specifying that we would like to have 50\% of men and 50\% of women in a committee, we could ask for a committee with at least 40\% of women and at least 40\% of men.

Having lower and upper quotas for attributes gives more flexibility and makes it more likely that a committee satisfying the constraints exists. However, if the number of attributes is large (for instance, hundreds or thousands) and the size of the candidate database is moderate, it is still likely that a committee satisfying all the constraints does not exist. Further, coming up with the constraints which, on the one hand are restrictive enough to implement multi-attribute proportionality to the extent that would be satisfactory, and on the other hand are liberal enough to ensure that a committee satisfying the constraints exists, is much less straightforward and requires more cognitive effort than simply providing a vector of ideal distributions.         

Interestingly, our results from \autoref{fpt} can be extended to the model with lower and upper quotas. Indeed, for the hardness it suffices to observe that the problem of finding a perfect committee can be easily formulated in the model with lower and upper quotas----it suffices to set the upper and lower quotas to the same value, equal to the value of the respective target distribution. For the positive result from \autoref{thm:fptParameterP} it suffices to change constraints (g) and (i) in the proof of the theorem so that variables $x_{i}^j$ and $y_{i}^j$ are compared against specific quotas instead of $k\pi_{i}^j$. Also the analysis from the proofs of \autoref{thm:localSearch1} and \autoref{thm:localSearch2} carries over to the case with quotas: in the proofs of these theorems one needs to define $X_{\mathrm{ex}}$ as the set of all attributes for which the analysed committee $A$ does not exceed the lower and upper bounds. Specifically, our local search algorithm would treat lower and upper quotas as soft constraints and would approximate the total violation of the constraints: 
\begin{align*}
\sum_{r_i^j(A) < \underline{\pi}_i^j}\left(\underline{\pi}_i^j - r_i^j(A)\right) + \sum_{r_i^j(A) > \overline{\pi}_i^j}\left(r_i^j(A) - \overline{\pi}_i^j\right)\text{,}
\end{align*}
where $\underline{\pi}_i^j$ and $\overline{\pi}_i^j$ denote the lower and upper quotas, respectively (cf. \autoref{def:hamilton_approximation}).

\subsection{Dependent Attributes}\label{discussion-dependent}

In our model we assume that the attributes are independent, which sometimes may lead to undesirable outcomes. For example, consider an instance where the goal is to select a committee consisting of 50\% of men and 50\% of women and of 50\% of junior and 50\% of senior people. In this instance, for $k=10$, a committee $A$ that consists of 5 junior men and 5 senior women is a perfect committee. However, junior women and senior men are clearly underrepresented in $A$. Another example is when our goal is to select a set of $k$ movies and when half of the population likes drama movies with Meryl Streep starring the main role and the other half likes action movies with Dwayne Johnson. A set of $\nicefrac{k}{2}$ action movies with Meryl Streep and $\nicefrac{k}{2}$ drama movies with Dwayne Johnson would form a perfect committee, even though it is incompatible with the voters' preferences\footnote{We thank the anonymous AIJ reviewer for suggesting this example.}. 

This phenomenon is known as the separability dilemma: 
\begin{itemize}
\item either preferences are assumed to be separable: in that case, they are cheap to communicate (and computing the outcome is generally easy); but it is a strong domain restriction. In our example, if we assumed that the preferences of the society expressed by the target distributions were separable, then a set with $\nicefrac{k}{2}$ action movies with Meryl Streep and $\nicefrac{k}{2}$ drama movies with Dwayne Johnson would form an excellent solution. 
\item or we don't make such an assumption and allow preferential dependencies between attributes. This increases the cost of communication exponentially in the worst case, and makes computation harder.
\end{itemize} 

Both approaches are often seen as too extreme, and the usual trade-off consists in allowing a reasonable amount of preferential dependencies. We can for instance introduce an artificial attribute combining some dependent attributes. For instance, in our first example we could introduce a combined attribute (Gender, Age) and we could require that there are 25\% of committee members representing each of the four values: (male, junior), (male, senior), (female, junior), and (female, senior). Since combining the attributes leads to an exponential growth of the length of the representation of the target distributions, this approach is only possible when the number of dependent attributes is relatively small (see also the discussion below \autoref{def:perfect_committee} in \autoref{model}, and the discussion on Constrained Approval Voting in \autoref{sec:constrained_av}).

\subsection{Other Metrics Measuring the Distance to the Target Distributions}\label{discussion-metric}

In \autoref{sec:multiAttributeRules} we defined the multi-attribute d'Hondt rule and the multi-attribute Hamilton rule in terms of minimisation or maximisation of sums of expressions. Another possibility is to define the $L^{\infty}$-multi-attribute d'Hondt rule as the one which outputs a committee $A$ maximising $\min_i \sum_{j} \pi_i^j\harmonic(r_i^j(A) \cdot k)$ and the $L^{\infty}$-multi-attribute Hamilton rule which outputs a committee $A$ minimising $\max_i \sum_{j} |r^j_i(A)-\pi^j_i|$. Both approaches have their advantages and shortcomings. For instance, with the $L^1$ metric it may happen that there exists an optimal committee which is far from the target distributions for half of the attributes while there exists another committee which violates the target distribution for each attribute, but to a significantly lower extent. Such a committee seems more appropriate in the context of proportional representation. On the other hand, if we follow an $L^{\infty}$-optimisation approach, it may happen that among a large number of attributes there exists a single ``outlier'' attribute $X_i$ with the target distribution set to $\pi_{i}^{1} = 1$ in spite of the fact that all candidates in the database have the value of this attribute equal to $x_i^0$. In such case a rule would select any committee, in particular it could select a committee which is far from the target distributions for every attribute, even though there might exist a committee which would be perfect for all attributes except for $X_i$. Naturally, there exist intermediate approaches---for instance, one could aim at maximising/minimising the $L^p$ norms of the appropriate expressions. 

The results from \autoref{fpt} easily extend to the case of otpmising $L^{\infty}$-aggregate. For instance, the ILPs from the proof of \autoref{thm:fptParameterP} can be naturally extended to the $L^{\infty}$-optimisation case, by using the standard constructions for implementing the ``max'' operator in the objective function. A natural question which remains open is whether the $L^{\infty}$-variants of our problems can be well approximated.    

\section{Conclusion}\label{conclu}

In this paper we have defined and studied multi-attribute generalisations of a well-known class of apportionment methods, in particular of the Hamilton and the d'Hondt methods of apportionment, albeit with motivations that go far beyond party-list elections (such as the selection of a collective set of candidates). We have formulated several axioms, commonly considered in the political science literature in the context of apportionment, for multi-attribute committee selection rules. Motivated with this axiomatic approach we have identified two multi-attribute committee selection rules that can be considered as extensions of the Hamilton and d'Hondt methods to multi-attribute scenarios.

We have studied the computational complexity of the problem of finding committees that, in some sense, best fit some given distribution of attribute values. We have found out that the problem is in general {\sf NP}-hard, but that it can be handled efficiently if the number of attributes is small. We have shown that the multi-attribute extensions of the Hamilton and d'Hondt methods can be well approximated. In particular, we have provided an interesting involved analysis of the local-search algorithm in the context of our multi-attribute setting. 

\subsection*{Acknowledgments}

We thank Eunjung Kim for giving us the initial idea of the paper, and for fruitful discussions. We further thank 
Katar\'{i}na Cechl\'{a}rov\'{a} for her comments on the relation between population and party population monotonicity.
J\'{e}r\^{o}me Lang was supported by the ANR project CoCoRICo-CoDec, project number ANR-14-CE24-0007.
Piotr Skowron was supported by the European Research Council grant ERC-StG 639945 (ACCORD) and by the Foundation for Polish Science within the Homing programme (Project title: "Normative Comparison of Multiwinner Election Rules").

\bibliographystyle{plain}
\bibliography{mpr-2}

\bigskip
\bigskip
\bigskip
\appendix
\section{Proofs Omitted from the Main Text}

\begin{repproposition}{prop:hamilton_party_population}
Under full supply property the Hamilton method satisfies party population monotonicity.
\end{repproposition}
\begin{proof}
Consider an instance $I$ of the apportionment problem, and let $I'$ be an instance obtained from $I$ by increasing the quota $\nicefrac{v_i}{v_+}$ for one party $P_i$, but leaving the ratios of quotas between the other parties unchanged. 
Let us consider the Hamilton method as the process that in steps allocates seats to parties (the first step is rounding down the quotas, and the next steps correspond to allocating seats to the parties in the descending order of their remainders). We want to prove that if this is the case that the Hamilton method assigns an $x$-th seat to $P_i$ before assigning a $y$-th seat to $P_j \neq P_i$ in $I$, then it is also the case in $I'$. This will show that the number of seats assigned to $P_i$ in $I'$ is at least as large as in $I$.

We know that the quota of $P_i$ in $I'$ is higher than in $I$. Also, for any other party $P_j \neq P_i$, we know that the quota of $P_j$ in $I'$ is lower than in $I$ (this is because the ratios of the quotas of the other parties remain unchanged; note that this argument would not work if we used population monotonicity instead of party population monotonicity). Thus, in the phase of rounding quotas down $P_i$ will get at least the same number of seats in $I'$ as in $I$. Also, if $P_j$ got the same number of seats after rounding in $I'$ as it got in $I$, then the remainder of $P_i$ is higher than the remainder of $P_j$ in $I'$ whenever it is the case that it was higher in $I$. Thus, if the Hamilton method assigned a seat to $P_i$ before $P_j$ in $I$, then it must also happen in $I'$.
\end{proof}

\begin{repproposition}{prop:singleAttributeProperties}
When $p = 1$ and assuming there are at least $k$ candidates for each value of the unique attribute, then $\calR_{\mathrm{H}}$ coincides with the Hamilton apportionment rule.
\end{repproposition}

\begin{proof}
Let $s_j^*$ denote the ideal number of seats for party $P_j$, i.e., $s_j^* = \pi^jk$. Let $A$ be a committee of size $k$ and let $R^j(A) = k \, r^j(A)$ be the number of members of $A$ that belong to party $P_j$. Since $ |R^j(A)-s_j^*| = k |r^j(A)-\pi^j|$, we need to show that the following two assertions are equivalent:
\begin{enumerate}
\item $A$ minimises $\sum_j |R^j(A)-s_j^*|$.
\item $A$ is a Hamilton committee.
\end{enumerate}
We first show $1 \Rightarrow 2$.
Assume $A$ is not a Hamilton committee: then there exists an attribute value (party) that receives strictly more or strictly less seats than it would receive according to the Hamilton method. Naturally, there must also exist an attribute that receives strictly less or strictly more seats, respectively. Formally, this means that there are two attribute values (parties), say $1$ and $2$, such that the target number of seats for parties 1 and 2 are $s_1^* = p+\alpha_1$ and $s_2^* = q+\alpha_2$, with $p, q 
$ being integers and $1 > \alpha_2 > \alpha_1 \geq 0$, and such that $R^1(A) \geq p+1$ and $R^2(A) \leq q$. 
We have
\begin{align*}
\sum_j |R^j(A)-\pi^j| &= \sum_{j \neq 1,2} |R^j(A)-s_j^*| + |R^1(A)-s_1^*| + |R^2(A)-s_2^*| \\
                      &\geq  \sum_{j \neq 1,2} |R^j(A)-s_j^*| + (1-\alpha_1) + \alpha_2 \text{.}
\end{align*}
Consider the committee $A'$ obtained from $A$ by giving one less seat to $1$ and one more to $2$ and consider the following three cases
\begin{description}
\item[Case 1:] If $R^1(A) > p+1$ then:
\begin{align*}
&\sum_j |R^j(A)-s_j^*| - \sum_j |R^j(A')-s_j^*| \\
& \hspace{1.5cm} = |R^1(A)-s_1^*| -  |R^1(A')-s_1^*| + |R^2(A)-s_2^*| - |R^2(A')-s_2^*| \\
& \hspace{1.5cm} \geq 1 + (1-\alpha_2)-\alpha_2 > 0.
\end{align*}
\item[Case 2:]
If $R^2(A) < q$ then similarly,  $\sum_j |R^j(A)-s_j^*| - \sum_j |R^j(A')-s_j^*|> 0$.  
\item[Case 3:]
If $R^1(A) = p+1$ and $R^2(A) = q$ then we have:
\begin{align*}
\sum_j |R^j(A)-s_j^*| = \sum_{j \neq 1,2} |R^j(A)-s_j^*| + (1-\alpha_1) + \alpha_2
\end{align*}
and
\begin{align*}
\sum_j |R^j(A')-s_j^*| = \sum_{j \neq 1,2} |R^j(A')-s_j^*| + (1-\alpha_2) + \alpha_1
\end{align*}
Hence:
\begin{align*}
\sum_j |R^j(A)-s_j^*| - \sum_j |R^j(A')-s_j^*| = 2(\alpha_2-\alpha_1) > 0.
\end{align*}
\end{description}
\noindent In all three cases,  $A$ does not minimise $\sum_j |R^j(A)-s_j^*|$, which gives a contradiction.

It remains to be shown that $2 \Rightarrow 1$, i.e., that if $A$ is a Hamilton committee then it minimises $\sum_j |R^j(A)-s_j^*|$. If there is a unique Hamilton committee then this follows immediately from $1  \Rightarrow 2$. Assume there are several Hamilton committees $A_1, \ldots, A_q$. Then there are $q$ parties, w.l.o.g., let us call them $P_1, \ldots, P_q$, with equal remainders $\alpha \in [0,1)$, that is, $s_1^* = p_1 + \alpha$, \ldots, $s_q^* = p_q + \alpha$, and these committees differ only with respect to whether they get an extra seat or not. We easily check that for any two $A, A'$ of these committees we have $\sum_j |R^j(A)-s_j^*| = \sum_j |R^j(A')-s_j^*|$.
\end{proof}


\begin{repproposition}{prop:multiAttributeProperties}
Under the full supply assumption, non-reversal, respect of quota, and value monotonicity with respect to every attribute are all satisfied by the multi-attribute Hamilton rule.
In the general case, non-reversal, and respect of quota are not satisfied. If $X_i$ is a binary variable, then value monotonicity with respect to $X_i$ is satisfied; however it is not satisfied in the general case.
\end{repproposition}

\begin{proof}
Under the full supply assumption, the result easily comes from \autoref{prop:fs}
and the fact that the property holds in the single-attribute case.

In the general case, we give counterexamples. For respect of quota, we have two binary attributes, and two candidates $a$, $b$ with value vectors $(x_1^2, x_2^2)$ and $(x_1^1,x_2^1)$, $k = 1$, $\pi$ defined  as $\pi_1^1 = 0$, $\pi_1^2 = 1$, $\pi_2^1 = 1$, $\pi_2^2 = 0$. The committee minimising our metric is either $\{a\}$ or $\{b\}$, and does not respect quota even though all values $k \pi_i^j$ are integers.

For non-reversal we have two binary attributes and six candidates: $a,b,c$, each with vector $(x_1^1, x_2^1)$ and $d,e,f$, each with vector $(x_1^2, x_2^2)$. We have a target distribution $\pi$ defined as follows: $\pi_1^1 = 0.35$, $\pi_1^2 = 0.65$, $\pi_2^1 = 1$, $\pi_2^2 = 0$. We set $k = 3$. The committees minimising our metric are $\{a,b,c\}$ and all triples made up from two candidates out of $\{a,b,c\}$ and one out of $\{d,e,f\}$. In all cases, we have $r_1^1(A) > r_1^2(A)$ even though $\pi_1^1 < \pi_1^2$.

Now, we prove that value monotonicity holds for binary domains. In the following we will use notation $\| r(A) - \pi\| = \sum_{i, j} |r_i^j(A) -  \pi_i^j|$.
Consider a binary attribute $X_i$, with $D_i = \{x_i^0, x_i^1\}$.
Assume that $\rho_{i}^0 > \pi_{i}^0$ (and so $\rho_{i}^{1} < \pi_{i}^{1}$), and that for all $i' \neq i$ we have $\rho_{i'} =  \pi_{i'}$. Let $A$ be an committee minimising our metric for $\pi$ and, for the sake of contradiction, assume that for all committees $B$ minimising our metric for $\rho$ we have $r_i^0(B) < r_i^0(A)$. Let $B$ be such a committee.
The proof is a case by case study, with six cases to be considered: (C1) $r_i^0(B) \leq \pi_i^0  < \rho_i^0 \leq r_i^0(A)$; (C2) $\pi_i^0 \leq r_i^0(B) \leq \rho_i^0 \leq r_i^0(A)$; (C3) $\pi_i^0 < \rho_i^0 \leq r_i^0(B) < r_i^0(A)$; (C4) $r_i^0(B) \leq \pi_i^0 \leq r_i^0(A) \leq \rho_i^0$; (C5) $\pi_i^0 \leq r_i^0(B) < r_i^0(A) \leq \rho_i^0$; and (C6) $r_i^0(B) < r_i^0(A) \leq \pi_i^0 < \rho_i^0$. 

\begin{itemize}
\item Case 1: $r_i^0(B) \leq \pi_i^0  < \rho_i^0 \leq r_i^0(A)$. In this case we have $r_i^{1}(B) \geq \pi_i^{1} > \rho_i^{1} \geq r_i^{1}(A)$ and the following holds:
 
$\begin{array}{lll} 
 \| r(B) - \pi\|\ & = \sum_{i' \neq i}\sum_{j}|r_{i'}^{j}(B) - \pi_{i'}^{j}| + (\pi_{i}^{0} - r_i^0(B)) + (r_i^{1}(B) - \pi_{i}^{1}) & (1) \\
& = \sum_{i' \neq i}\sum_{j}|r_{i'}^{j}(B) - \rho_{i'}^{j}| + (\rho_{i}^{0} - r_i^0(B)) + (r_i^{1}(B) - \rho_{i}^{1}) & \\
& \;\;\;\;\;\;\;\; + \pi_{i}^{0} - \pi_{i}^{1} - \rho_{i}^{0} + \rho_{i}^{1} & (2) \\
& = \| r(B) - \rho \|\ + 2(\pi_i^0 - \rho_i^0) & (3)\\
& <  \| r(A) - \rho \|\ + 2(\pi_i^0 - \rho_i^0) & (4)\\
& = \sum_{i' \neq i}\sum_{j}|r_{i'}^{j}(A) - \rho_{i'}^{j}| + (r_i^0(A) - \rho_{i}^{0}) + (\rho_{i}^{1} - r_i^{1}(A)) + 2(\pi_i^0 - \rho_i^0) & (5) \\
& = \sum_{i' \neq i}\sum_{j}|r_{i'}^{j}(A) - \rho_{i'}^{j}| + (r_i^0(A) - \pi_{i}^{0}) + (\pi_{i}^{1} - r_i^{1}(A)) &\\
& \;\;\;\;\;\;\;\; + \pi_{i}^{0} - \pi_{i}^1 - \rho_{i}^{0} + \rho_{i}^1 + 2(\pi_i^0 - \rho_i^0) & (6) \\
& = \| r(A) - \pi \|\ + 4(\pi_i^0 - \rho_i^0) & (7)  \\
& \leq \| r(A) - \pi \|\ & (8) 
\end{array}$

\noindent (4) comes from the fact that $A$ does not minimise $f$ for $\rho$. Since, there is one strong inequality in the sequence, we imply that $A$ does not minimise $f$ for $\pi$, a contradiction. 

\item Case 2: $\pi_i^0 \leq r_i^0(B) \leq \rho_i^0 \leq r_i^0(A)$.

$\begin{array}{ll} 
 \| r(B) - \pi\|\ & = \sum_{i' \neq i}\sum_{j}|r_{i'}^{j}(B) - \pi_{i'}^{j}| + (r_i^0(B) - \pi_{i}^{0}) + (\pi_{i}^1 - r_i^1(B)) \\
& = \sum_{i' \neq i}\sum_{j}|r_{i'}^{j}(B) - \rho_{i'}^{j}| + (\rho_{i}^{0} - r_i^0(B)) + (r_i^1(B) - \rho_{i}^1) \\
& \;\;\;\;\;\;\;\; + 2r_i^0(B) - \pi_{i}^{0} - \rho_{i}^{0} - 2r_i^1(B) + \pi_{i}^1 + \rho_{i}^1 \\
& = \| r(B) - \rho \|\ + 4r_i^0(B)  - 2\pi_{i}^{0} - 2\rho_{i}^{0} \\
& < \| r(A) - \rho \|\ + 4r_i^0(B)  - 2\pi_{i}^{0} - 2\rho_{i}^{0} \\
& = \sum_{i' \neq i}\sum_{j}|r_{i'}^{j}(A) - \rho_{i'}^{j}| + (r_i^0(A) - \rho_{i}^{0}) + (\rho_{i}^1 - r_i^1(A)) + 4r_i^0(B)  - 2\pi_{i}^{0} - 2\rho_{i}^{0} \\
& = \sum_{i' \neq i}\sum_{j}|r_{i'}^{j}(A) - \rho_{i'}^{j}| + (r_i^0(A) - \pi_{i}^{0}) + (\pi_{i}^1 - r_i^1(A)) \\
& \;\;\;\;\;\;\;\; + \pi_{i}^{0} - \rho_{i}^{0} - \pi_{i}^1 + \rho_{i}^1 + 4r_i^0(B)  - 2\pi_{i}^{0} - 2\rho_{i}^{0} \\
& = \| r(A) - \pi \| + 4r_i^0(B) - 4\rho_{i}^{0} \\
& \leq \| r(A) - \pi \|
\end{array}$

Again we obtain a contradiction.

\item Case 3: $\pi_i^0 < \rho_i^0 \leq r_i^0(B) < r_i^0(A)$.

$\begin{array}{ll} 
 \| r(B) - \pi\|\ & = \sum_{i' \neq i}\sum_{j}|r_{i'}^{j}(B) - \pi_{i'}^{j}| + (r_i^0(B) - \pi_{i}^{0}) + (\pi_{i}^1 - r_i^1(B)) \\
& = \sum_{i' \neq i}\sum_{j}|r_{i'}^{j}(B) - \rho_{i'}^{j}| + (r_i^0(B) - \rho_{i}^{0}) + (\rho_{i}^1 - r_i^1(B)) \\
& \;\;\;\;\;\;\;\; - \pi_{i}^{0} + \rho_{i}^{0} + \pi_{i}^1 - \rho_{i}^1 \\
& = \| r(B) - \rho \|\ - 2\pi_{i}^{0} + 2\rho_{i}^{0} \\
& < \| r(A) - \rho \|\ - 2\pi_{i}^{0} + 2\rho_{i}^{0} \\
& = \sum_{i' \neq i}\sum_{j}|r_{i'}^{j}(A) - \rho_{i'}^{j}| + (r_i^0(A) - \rho_{i}^{0}) + (\rho_{i}^1 - r_i^1(A)) - 2\pi_{i}^{0} + 2\rho_{i}^{0} \\
& = \sum_{i' \neq i}\sum_{j}|r_{i'}^{j}(A) - \rho_{i'}^{j}| + (r_i^0(A) - \pi_{i}^{0}) + (\pi_{i}^1 - r_i^1(A)) \\
& \;\;\;\;\;\;\;\; + \pi_{i}^{0} - \rho_{i}^{0} - \pi_{i}^1 + \rho_{i}^1 - 2\pi_{i}^{0} + 2\rho_{i}^{0} \\
& = \| r(A) - \pi \|
\end{array}$

\item Case  4: $r_i^0(B) \leq \pi_i^0 \leq r_i^0(A) \leq \rho_i^0$.

$\begin{array}{ll} 
 \| r(B) - \pi\|\ & = \sum_{i' \neq i}\sum_{j}|r_{i'}^{j}(B) - \pi_{i'}^{j}| + (\pi_{i}^{0} - r_i^0(B)) + (r_i^1(B) - \pi_{i}^1) \\
& = \sum_{i' \neq i}\sum_{j}|r_{i'}^{j}(B) - \rho_{i'}^{j}| + (\rho_{i}^{0} - r_i^0(B)) + (r_i^1(B) - \rho_{i}^1) \\
& \;\;\;\;\;\;\;\; \pi_{i}^{0} - \rho_{i}^{0} - \pi_{i}^1 + \rho_{i}^1 \\
& = \| r(B) - \rho \|\ + 2\pi_{i}^{0} - 2\rho_{i}^{0} \\
& < \| r(A) - \rho \|\ + 2\pi_{i}^{0} - 2\rho_{i}^{0} \\
& = \sum_{i' \neq i}\sum_{j}|r_{i'}^{j}(A) - \rho_{i'}^{j}| + (\rho_{i}^{0} - r_i^0(A)) + (r_i^1(A) - \rho_{i}^1) + 2\pi_{i}^{0} - 2\rho_{i}^{0} \\
& = \sum_{i' \neq i}\sum_{j}|r_{i'}^{j}(A) - \rho_{i'}^{j}| + (r_i^0(A) - \pi_{i}^{0}) + (\pi_{i}^1 - r_i^1(A)) \\
& \;\;\;\;\;\;\;\; - 2r_i^0(A) + 2r_i^1(A) + \pi_{i}^{0} + \rho_{i}^{0} - \pi_{i}^1 - \rho_{i}^1 + 2\pi_{i}^{0} - 2\rho_{i}^{0} \\
& =  \| r(A) - \pi \| - 4r_i^0(A) + 4\pi_{i}^{0} \\
& \leq \| r(A) - \pi \|
\end{array}$

\item Case 5: $\pi_i^0 \leq r_i^0(B) < r_i^0(A) \leq \rho_i^0$.

$\begin{array}{ll} 
 \| r(B) - \pi\|\ & = \sum_{i' \neq i}\sum_{j}|r_{i'}^{j}(B) - \pi_{i'}^{j}| + (r_i^0(B) - \pi_{i}^{0}) + (\pi_{i}^1 - r_i^1(B)) \\
& = \sum_{i' \neq i}\sum_{j}|r_{i'}^{j}(B) - \rho_{i'}^{j}| + (\rho_{i}^{0} - r_i^0(B)) + (r_i^1(B) - \rho_{i}^1) \\
& \;\;\;\;\;\;\;\; + 2r_i^0(B) - 2r_i^1(B)  - \pi_{i}^{0} - \rho_{i}^{0} + \pi_{i}^1 + \rho_{i}^1 \\
& = \| r(B) - \rho \|\ + 4r_i^0(B) - 2\pi_{i}^{0} - 2\rho_{i}^{0} \\
& < \| r(A) - \rho \|\ + 4r_i^0(B) - 2\pi_{i}^{0} - 2\rho_{i}^{0} \\
& = \sum_{i' \neq i}\sum_{j}|r_{i'}^{j}(A) - \rho_{i'}^{j}| + (\rho_{i}^{0} - r_i^0(A)) + (r_i^1(A) - \rho_{i}^1) + 4r_i^0(B) - 2\pi_{i}^{0} - 2\rho_{i}^{0} \\
& = \sum_{i' \neq i}\sum_{j}|r_{i'}^{j}(A) - \rho_{i'}^{j}| + (r_i^0(A) - \pi_{i}^{0}) + (\pi_{i}^1 - r_i^1(A)) \\
& \;\;\;\;\;\;\;\; + 4r_i^0(B) - 2r_i^0(A) + 2r_i^1(A) + \pi_{i}^{0} + \rho_{i}^{0} - \pi_{i}^1 - \rho_{i}^1 - 2\pi_{i}^{0} - 2\rho_{i}^{0} \\
& =  \| r(A) - \pi \| + 4r_i^0(B) - 4r_i^0(A) \\
& \leq \| r(A) - \pi \|
\end{array}$

\item Case  6: $r_i^0(B) < r_i^0(A) \leq \pi_i^0 < \rho_i^0$.

$\begin{array}{ll} 
 \| r(B) - \pi\|\ & = \sum_{i' \neq i}\sum_{j}|r_{i'}^{j}(B) - \pi_{i'}^{j}| + (\pi_{i}^{0} - r_i^0(B)) + (r_i^1(B) - \pi_{i}^1) \\
& = \sum_{i' \neq i}\sum_{j}|r_{i'}^{j}(B) - \rho_{i'}^{j}| + (\rho_{i}^{0} - r_i^0(B)) + (r_i^1(B) - \rho_{i}^1) \\
& \;\;\;\;\;\;\;\;  + \pi_{i}^{0} - \rho_{i}^{0} - \pi_{i}^1 + \rho_{i}^1 \\
& = \| r(B) - \rho \|\ + 2\pi_{i}^{0} - 2\rho_{i}^{0} \\
& < \| r(A) - \rho \|\ + 2\pi_{i}^{0} - 2\rho_{i}^{0} \\
& = \sum_{i' \neq i}\sum_{j}|r_{i'}^{j}(A) - \rho_{i'}^{j}| + (\rho_{i}^{0} - r_i^0(A)) + (r_i^1(A) - \rho_{i}^1) + 2\pi_{i}^{0} - 2\rho_{i}^{0} \\
& = \sum_{i' \neq i}\sum_{j}|r_{i'}^{j}(A) - \rho_{i'}^{j}| + (\pi_{i}^{0} - r_i^0(A)) + (r_i^1(A) - \pi_{i}^1) \\
& \;\;\;\;\;\;\;\; - \pi_{i}^{0} + \rho_{i}^{0} + \pi_{i}^1 - \rho_{i}^1 + 2\pi_{i}^{0} - 2\rho_{i}^{0} \\
& =  \| r(A) - \pi \|
\end{array}$
\end{itemize}

\noindent
Finally, we give an example showing that value monotonicity does not hold in the general case. First, we describe the set of attributes. We have one distinguished attribute $X_1$ with 5 possible values $x_1^1$, $x_1^2$, $x_1^3$, $x_1^4$, and $x_1^5$ and 64 groups of binary attributes, indexed with the pairs of integers $i, j \in \{1, 2, 3, 4\}$. These groups of attributes are denoted as $X_{(1, 1)}, X_{(1, 2)}, \ldots X_{(1, 8)}, X_{(2, 1)}, \ldots X_{(8, 8)}$. Each group contains some large number $\lambda$ of indistinguishable attributes, each having the same set of possible values $\{x_2^1, x_2^2\}$.
We have 16 alternatives $A_1, A_2, \ldots, A_8$, and $B_1, B_2, \ldots B_8$, and our goal is to select a subset of $k=8$ of them.

We start with describing these alternatives on binary attributes: each alternative $A_i$ has the value $x_2^1$ on all attributes $X_{(i, \cdot)}$ and the value $x_2^2$ on all the remaining ones; each alternative $B_i$ has the value $x_2^1$ on all attributes $X_{(\cdot, i)}$ and the value $x_2^2$ on all the remaining ones. For the binary attributes we set the target probabilities to $\pi_2^1 = \nicefrac{1}{8}$ and $\pi_2^2 = \nicefrac{7}{8}$. Due to this construction, we see that the only two subsets that perfectly agree with target distributions on each of binary attributes are $A = \{A_1, A_2, \ldots, A_8\}$ and $B = \{B_1, B_2, \ldots, B_8\}$. Indeed, every subset $S$ including $A_i$ and $B_j$, would have $r(S) \geq \nicefrac{1}{4}$ at least for one group of attributes $X_{(i, j)}$. Since $\lambda$ is large, we infer that, independently of what happens on the distinguished attribute $X_1$, the only possible winning committee is either $A = \{A_1, A_2, \ldots, A_8\}$ or $B = \{B_1, B_2, \ldots, B_8\}$.

Next, let us describe what happens on the attribute $X_1$. The vector $\langle r_1^j(A) \rangle$ is equal to $\langle r_1^j(A) \rangle = (\nicefrac{1}{2}, 0, \nicefrac{1}{2}, 0, 0)$. For the committee $B$, we have $\langle r_1^j(B) \rangle = (\nicefrac{1}{4}, \nicefrac{1}{4}, \nicefrac{1}{4}, \nicefrac{1}{8}, \nicefrac{1}{8})$, and the vector of target distributions for $X_1$ is equal $\pi_1 = (0, 0, \nicefrac{3}{8} + \epsilon, \nicefrac{5}{8} -\epsilon, 0)$. We can see that $\| r(A) - \pi\|\ = \nicefrac{1}{2} + \nicefrac{1}{8} - \epsilon + \nicefrac{5}{8} -\epsilon = 1.25 - 2\epsilon$. Since, $\| r(B) - \pi\|\ = \nicefrac{1}{4} + \nicefrac{1}{4} + \nicefrac{1}{8} + \epsilon + \nicefrac{4}{8} -\epsilon + \nicefrac{1}{8} = 1.25$, we get that $A$ is a winning committee. However, if we modify the target fractions so that $\rho_1 = (\nicefrac{1}{4}, 0, \nicefrac{9}{32} + \epsilon_1, \nicefrac{15}{32} -\epsilon_2, 0)$, we will get $\| r(A) - \rho\|\ = \nicefrac{1}{4} + \nicefrac{7}{32} - \epsilon_1 + \nicefrac{15}{32} -\epsilon_2 = \nicefrac{30}{32} - \epsilon_1 - \epsilon_2$ and $\| r(B) - \rho\|\ = \nicefrac{1}{4} + \nicefrac{1}{32} + \epsilon_1 + \nicefrac{11}{32} - \epsilon_2 + \nicefrac{1}{8} = \nicefrac{24}{32} + \epsilon_1 - \epsilon_2$, thus, $B$ is winning according to $\rho$. However, $B$ has lower representation of $x_1^1$ than $A$, and $\rho$ was obtained from $\pi$, by increasing the fraction of $\pi_1^1$. This completes the proof. 
\end{proof}

\begin{repproposition}{prop:fsdhondt}
Consider a candidate database that satisfies the full supply property. For any attribute $X_i$, any committee $A$ that maximises $\sum_{i, j} \pi_i^j\harmonic(r_i^j(A) \cdot k)$ is a d'Hondt committee for the single-attribute problem $(\{X_i\},D^{\downarrow X_i}, \pi_i, k)$, where  $D^{\downarrow X_i}$ is the projection of $D$ on $\{X_i\}$.
\end{repproposition}
\begin{proof}
The idea from the proof of \autoref{prop:fs} works also for this proposition. If there exists a committee $A$ which maximises $\sum_{i}\sum_{j} \pi_i^j\harmonic(r_i^j(A) \cdot k)$ and which is not a d'Hondt committee for the single-attribute problem $(\{X_i\},D^{\downarrow X_i}, \pi_i, k)$, then by \autoref{prop:dhondt_alternative_definition} there exists a committee $B$ such that \mbox{$\sum_{j} \pi_i^j\harmonic(r_i^j(B) \cdot k) > \sum_{j} \pi_i^j\harmonic(r_i^j(A) \cdot k)$}. Similarly as in the proof of \autoref{prop:fs}, it is possible to build a committee $D$ from $A$ and $B$ such that $\sum_{i}\sum_{j} \pi_i^j\harmonic(r_i^j(D) \cdot k) > \sum_{i}\sum_{j} \pi_i^j\harmonic(r_i^j(A) \cdot k)$, which gives a contradiction and completes the proof.  
\end{proof}

\begin{repproposition}{prop:multiAttributePropertiesdHondt}
Under the full supply assumption, non-reversal, house monotonicity, and value monotonicity with respect to every attribute are all satisfied by the multi-attribute d'Hondt method.
In the general case, non-reversal and house monotonicity are not satisfied. If $X_i$ is a binary variable, then value monotonicity with respect to $X_i$ is satisfied;
however it is not satisfied in the general case.
\end{repproposition}

\begin{proof}
Similarly as in the proof of \autoref{prop:multiAttributeProperties} we infer that the result for full supply assumption follows from \autoref{prop:fsdhondt} and from the fact that the respective properties holds in the single-attribute case.

In the general case, we give counterexamples. For non-reversal, the same example as in the proof of \autoref{prop:multiAttributeProperties} works also for the case of the multi-attribute d'Hondt method.

For house monotonicity we have two binary attributes and three candidates: $a$ with vector $(x_1^1, x_2^2)$, $b$ with vector $(x_1^2, x_2^1)$, and $c$ with vector $(x_1^2, x_2^2)$. We have a target distribution $\pi$ defined as follows: $\pi_1^1 = \pi_2^1 = 0.5 - \epsilon$ and $\pi_1^2 = \pi_2^2 = 0.5 + \epsilon$, for some small positive $\epsilon$. For $k = 1$ candidate $c$ should be selected, while for $k = 2$ committee $\{a, b\}$ is optimal.

Now, we prove that value monotonicity holds for binary domains.
Consider a binary attribute $X_i$, with $D_i = \{x_i^0, x_i^1\}$.
Assume that $\rho_{1}^0 > \pi_{1}^0$, and that for all $i' \neq i$ we have $\rho_{i'} =  \pi_{i'}$. Let $A$ be an committee maximising our metric for $\pi$ and, for the sake of contradiction, assume that for all committees $B$ maximising our metric for $\rho$ we have $r_i^0(B) < r_i^0(A)$. Let $B$ be such a committee.

\begin{align*}
&\sum_{i}\sum_{j} \pi_i^j\harmonic(r_i^j(B) \cdot k) = \sum_{i \neq 1}\sum_{j}\pi_i^j\harmonic(r_i^j(B) \cdot k) + \pi_1^0 \harmonic(r_1^0(B) \cdot k) + (1 - \pi_1^0)\harmonic(k - r_1^0(B) \cdot k) \\
         & \hspace{1cm} = \sum_{i \neq 1}\sum_{j}\rho_i^j\harmonic(r_i^j(B) \cdot k) + \rho_1^0 \harmonic(r_1^0(B) \cdot k) + (1 - \rho_1^0)\harmonic(k - r_1^0(B) \cdot k) \\
         & \hspace{2cm} + (\rho_1^0 - \pi_1^0) \cdot (\harmonic(k - r_1^0(B) \cdot k) - \harmonic(r_1^0(B) \cdot k)) \\
         & \hspace{1cm} > \sum_{i}\sum_{j} \rho_i^j\harmonic(r_i^j(A) \cdot k) + (\rho_1^0 - \pi_1^0) \cdot (\harmonic(k - r_1^0(B) \cdot k) - \harmonic(r_1^0(B) \cdot k)) \\
         & \hspace{1cm} \geq \sum_{i}\sum_{j} \rho_i^j\harmonic(r_i^j(A) \cdot k) + (\rho_1^0 - \pi_1^0) \cdot (\harmonic(k - r_1^0(A) \cdot k) - \harmonic(r_1^0(A) \cdot k)) \\
         & \hspace{1cm} = \sum_{i \neq 1}\sum_{j}\rho_i^j\harmonic(r_i^j(A) \cdot k) + \rho_1^0 \harmonic(r_1^0(A) \cdot k) + (1 - \rho_1^0)\harmonic(k - r_1^0(A) \cdot k) \\
         & \hspace{2cm} + (\rho_1^0 - \pi_1^0) \cdot (\harmonic(k - r_1^0(A) \cdot k) - \harmonic(r_1^0(A) \cdot k)) \\
         & \hspace{1cm} = \sum_{i \neq 1}\sum_{j}\pi_i^j\harmonic(r_i^j(A) \cdot k) + \pi_1^0 \harmonic(r_1^0(A) \cdot k) + (1 - \pi_1^0)\harmonic(k - r_1^0(A) \cdot k) \\
         & \hspace{1cm} = \sum_{i}\sum_{j} \pi_i^j\harmonic(r_i^j(A) \cdot k) \text{.}
\end{align*} 
We get that $B$ is better with respect to our metric than $A$ for $\pi$, a contradiction.

Finally, from the proof of \autoref{prop:multiAttributeProperties} we can reuse parts of the construction showing that value monotonicity does not hold in the general case. Let us recall that the construction there ensures that one of the two committees, $A = \{A_1, A_2, \ldots, A_8\}$ or $B = \{B_1, B_2, \ldots, B_8\}$, needs to be selected. Additionally we can have two attributes, $X_1$ and $X_2$, each with three possible values. These two attributes determine whether $A$ or $B$ is going to be selected. We select $A$ and $B$ so that:
\begin{align*}
&\langle r_1^j(A) \rangle = (0, 1, 0) \quad & \langle r_1^j(B) \rangle = (0, 0, 1) \\
&\langle r_2^j(A) \rangle = (\nicefrac{1}{8}, \nicefrac{7}{8}, 0) \quad & \langle r_1^2(B) \rangle = (0, 0, 1)\textrm{.}
\end{align*}
We set $\pi_1 = (0, 0.1, 0.9)$. Now, consider $\pi_2 = (0, 1, 0)$. For the two attributes the values of committees $A$ and $B$ are equal to:
\begin{align*}
\text{committee $A$}\colon  0.1 \cdot \harmonic(8) + \harmonic(7) \approx 2.86 \text{,} \quad
\text{committee $B$}\colon  0.9 \cdot \harmonic(8) \approx 2.44 \text{.}
\end{align*}
Consequently, $A$ will be selected by the multi-attribute d'Hondt method. Now, consider what happens when we change $\pi_2$ to $\rho_2 = (1, 0, 0)$. For the two attributes the values of committees $A$ and $B$ are now equal to: 
\begin{align*}
\text{committee $A$}\colon  0.1 \cdot \harmonic(8) + \harmonic(1) \approx 1.27 \text{,} \quad
\text{committee $B$}\colon  0.9 \cdot \harmonic(8) \approx 2.44 \text{.} 
\end{align*}
Yet, $B$ has lower representation of $x_2^0$ than $A$, and $\rho$ was obtained from $\pi$, by increasing the fraction of $\pi_2^0$. This shows that value monotonicity is not satisfied in the general case and completes the proof. 
\end{proof}

\end{document}